\documentclass{article}

\usepackage{microtype}
\usepackage{graphicx}
\usepackage{subfigure}
\usepackage{booktabs} 
\usepackage{amsfonts,amssymb}
\usepackage{amsmath}
\usepackage{amsthm}
\usepackage{lipsum}
\usepackage{listings}
\usepackage{nicefrac}
\usepackage{enumitem}
\usepackage{dsfont}
\usepackage{multirow}
\usepackage{setspace}
\usepackage{xcolor}
\def\##1\#{\begin{align}#1\end{align}}
\def\$#1\${\begin{align*}#1\end{align*}}

\newtheorem{remark}{Remark}

\newtheorem{thm}{Theorem}

\def\given{\,|\,}
\def\cS{{\mathcal{S}}}
\def\cD{{\mathcal{D}}}
\def\cA{{\mathcal{A}}}
\def\cN{{\mathcal{N}}}
\DeclareMathOperator*{\argmax}{argmax}
\DeclareMathOperator*{\argmin}{argmin}

\usepackage{hyperref}

\usepackage[accepted]{icml2021}

\icmltitlerunning{Principled Exploration via Optimistic Bootstrapping and Backward Induction}

\begin{document}

\twocolumn[
\icmltitle{Principled Exploration via Optimistic Bootstrapping and Backward Induction}




\icmlsetsymbol{equal}{*}

\begin{icmlauthorlist}
\icmlauthor{Chenjia Bai}{hit}
\icmlauthor{Lingxiao Wang}{nw}
\icmlauthor{Lei Han}{tencent}
\icmlauthor{Jianye Hao}{tju}
\icmlauthor{Animesh Garg}{uot}
\icmlauthor{Peng Liu}{hit}
\icmlauthor{Zhaoran Wang}{nw}
\end{icmlauthorlist}

\icmlaffiliation{hit}{Harbin Institute of Technology, Harbin, China}
\icmlaffiliation{nw}{Northwestern University, Evanston, USA}
\icmlaffiliation{tencent}{Tencent Robotics X}
\icmlaffiliation{tju}{Tianjin University}
\icmlaffiliation{uot}{University of Toronto, Vector Institute}
\icmlcorrespondingauthor{Chenjia Bai}{bai\_chenjia@stu.hit.edu.cn}

\icmlkeywords{}

\vskip 0.3in
]



\printAffiliationsAndNotice{} 

\begin{abstract}
One principled approach for provably efficient exploration is incorporating the upper confidence bound (UCB) into the value function as a bonus. However, UCB is specified to deal with linear and tabular settings and is incompatible with Deep Reinforcement Learning (DRL). In this paper, we propose a principled exploration method for DRL through Optimistic Bootstrapping and Backward Induction (OB2I). OB2I constructs a general-purpose UCB-bonus through non-parametric bootstrap in DRL. The UCB-bonus estimates the epistemic uncertainty of state-action pairs for optimistic exploration. We build theoretical connections between the proposed UCB-bonus and the LSVI-UCB in a linear setting. We propagate future uncertainty in a time-consistent manner through episodic backward update, which exploits the theoretical advantage and empirically improves the sample-efficiency. Our experiments in the MNIST maze and Atari suite suggest that OB2I outperforms several state-of-the-art exploration approaches.
\end{abstract}

\section{Introduction}

In Reinforcement learning (RL)~\citep{sutton-2018}, an agent aims to maximize the long-term return by interacting with an unknown environment. 
To find the optimal policy, the agent is required to sufficiently explore the unknown environment and exploit in depth along the optimal trajectory. Devising efficient exploration algorithms thus becomes an attractive topic in recent years of RL research. The theoretical achievements in RL offer various provably efficient exploration methods in tabular and linear Markov Decision Processes (MDPs) based on the fundamental value iteration algorithm Least-Squares Value Iteration (LSVI). Among these, \textit{optimism in the face of uncertainty}~\citep{auer-2007,jin-2018} is a principled approach for efficient exploration with well theoretical guarantees. In tabular cases, the optimism-based methods incorporate the Upper Confidence Bound~(UCB) into the value function as bonus and attain the optimal worst-case regret~\citep{minmax-2017,regret-2010,dann-2015}. Randomized value function based on posterior sampling chooses actions according to the randomly sampled statistically plausible value function and is known to achieve near-optimal worst-case and Bayesian regrets~\citep{osband-2017,russo-2019}. Recently, the theoretical analyses in tabular cases have been extended to linear MDPs where the transition and reward function are assumed to be linear. In linear cases, LSVI-UCB~\citep{jin-2019} has been demonstrated to enjoy a near-optimal worst-case regret using a provably efficient bonus. Randomized LSVI~\citep{zanette-2019} also obtains a near-optimal worst-case regret. 

Although the analyses in tabular and linear cases have induced attractive approaches for efficient exploration, it is still challenging in developing a practical exploration algorithm that is essentially suitable for Deep Reinforcement Learning (DRL)~\citep{DQN-2015}, which is necessary to achieve human-level performance in large-scale tasks such as Atari games and robotic tasks. A simple evidence is that, in linear case, the bonus in LSVI-UCB~\citep{jin-2019} and nontrivial noise in randomized LSVI~\citep{zanette-2019} are specifically designed for linear models~\citep{abbasi-2011}, without generalizations to fit powerful function approximations such as neural networks.   

In this paper, we propose a principled exploration method for DRL through Optimistic Bootstrapping and Backward Induction (OB2I). OB2I is an instantiation of LSVI-UCB~\citep{jin-2019} in DRL by using a general-purpose UCB-bonus to provide an optimistic $Q$-value and a randomized value function to perform temporally-extended exploration. This general-purpose UCB-bonus represents the disagreement of bootstrapped $Q$-functions~\citep{bootstrap-2016} to measure the epistemic uncertainty of the unknown optimal value function. Importantly, this proposed UCB-bonus can also be theoretically demonstrated to be equivalent to the bonus-term in LSVI-UCB~\citep{jin-2019}, when moving back in linear MDPs. In our case, the $Q$-value plus the general-purpose UCB-bonus is shown to be an optimistic $Q^+$ function that is higher than the $Q$-value for scarcely visited state-action pairs and remains close to the $Q$-value for frequently visited pairs. Furthermore, we propose an extension of the Episodic Backward Update (EBU) technique \citep{ebu-2019}, which we refer to as Backward Induction, to propagate future uncertainties to the estimated action-value function consistently within an episode. The Backward Induction exploits the theoretical advantage of LSVI-UCB and empirically improves the sample-efficiency of exploration significantly.

Compared to existing count-based and curiosity-driven exploration methods~\citep{bonux-2020}, OB2I enjoys the following benefits: (\romannumeral 1) we utilize intrinsic rewards to produce optimistic value function and also take advantage of bootstrapped $Q$-learning to perform temporally-consistent exploration, while existing methods do not combine these two principles; (\romannumeral 2) the generalized UCB-bonus measures the disagreement of bootstrapped $Q$-values, considering long-term uncertainty in an episode rather than the single-step uncertainty used in most bonus-based methods~\citep{disagree-2019,RND-2019}; (\romannumeral 3) we provide theoretical analysis showing that OB2I is consistent to LSVI-UCB in linear case; (\romannumeral 4) extensive evaluations show that OB2I outperforms several strong exploration approaches in the MNIST maze game and 49 Atari games.

\section{Background}

In this section, we review bootstrapped DQN~\citep{bootstrap-2016} and LSVI-UCB~\citep{jin-2019} that are closely related to the proposed OB2I method.

\subsection{Bootstrapped DQN}

Considering an MDP represented as $(\mathcal{S},\mathcal{A},T,\mathbb{P},r)$, where $T\in\mathbb{Z}_{+}$ is the episode length, $\mathcal{S}$ is the state space, $\mathcal{A}$ is the action space, $r$ is the reward function, and $\mathbb{P}$ is the unknown dynamics. In each timestep, the agent observes the current state $s_t$ and takes an action $a_t$, and then it receives a reward $r_t$ and the next state $s_{t+1}$. The action-value function $Q^{\pi}(s_t,a_t):=\mathbb{E}_{\pi}\big[\sum_{i=t}^{T-1}{\gamma^{i-t} r_i}]$ represents the expected cumulative reward starting from state $s_t$ by taking action $a_t$ and following policy $\pi(a_t|s_t)$ until the end of the episode. $\gamma\in[0,1)$ is the discount factor. The optimal value function $Q^*=\max_{\pi}Q^{\pi}$, and the optimal action $a^*=\argmax_{a\in \mathcal{A}}Q^*(s,a)$.


Bootstrapped DQN~\citep{bootstrap-2016,bootstrap-2018} is a non-parametric posterior sampling method, which maintains $K$ estimations of $Q$-values to represent the posterior distribution of the randomized value function. Bootstrapped DQN uses a multi-head network with a shared representation and $K$ heads. Each head defines a $Q^k$-function. Bootstrapped DQN diversifies different $Q^k$ by using different random initialization and individual target networks. The loss for training $Q^k$ is
\begin{equation}\nonumber
L(\theta^k)\!=\!\mathbb{E}\Bigl[\bigl(r_t+\gamma\max_{a'}Q^k(s_{t+1},a';\theta^{k-})-Q^k(s_t,a_t;\theta^k)\bigr)^2\Bigr].
\end{equation}
The $k$-th head $Q^k(s,a;\theta^k)$ is trained with its own target network $Q^k(s,a;\theta^{k-})$ with slow-moving parameter $\theta^{k-}$. The agent follows a sampled head $Q^k$ to choose actions in an entire episode, which provides temporally-consistent exploration for DRL.

\subsection{LSVI-UCB}

\begin{algorithm}[t]
\small
\caption{LSVI-UCB in linear MDP}
\label{alg1}
\begin{algorithmic}[1]
\STATE {{\bf Initialize:} $\Lambda_t\leftarrow\lambda\cdot\mathbf{I}$ and $w_h\leftarrow 0$}
\FOR {episode $m=0$ {\bfseries to} $M-1$}
\STATE {Receive the initial state $s_0$}\label{alg:interact}
\FOR {step $t=0$ {\bfseries to} $T-1$}
\STATE {Take action $a_t=\arg\max_{a}Q_t(s_t,a)$ and observe $s_{t+1}$}
\ENDFOR\label{alg:end-interact}
\FOR {step $t=T-1$ {\bfseries to} $0$}\label{alg:train}
\STATE {$\Lambda_t\leftarrow\sum_{\tau=0}^{m}\phi(x_t^{\tau},a_t^{\tau})\phi(x_t^{\tau},a_t^{\tau})^\top+\lambda \cdot \mathrm{\mathbf{I}}$}
\STATE {$w_t\leftarrow\Lambda_t^{-1}\sum_{\tau=0}^{m}\phi(x_t^{\tau},a_t^{\tau})[r_t(x_t^{\tau},a_t^{\tau})+\max_a Q_{t+1}(x_{t+1}^{\tau},a)]$}
\STATE {$Q_t(\cdot,\cdot)\!=\!\min\{w_t^\top\phi(\cdot,\cdot)\!+\!\alpha[\phi(\cdot,\cdot)^\top\Lambda_t^{-1}\phi(\cdot,\cdot)]^{\nicefrac{1}{2}},T\}$}
\ENDFOR\label{alg:end-train}
\ENDFOR
\end{algorithmic}
\end{algorithm}

LSVI-UCB~\citep{jin-2019} uses an optimistic $Q$-value with LSVI in linear MDP. We denote the feature map of the state-action pair as $\phi:\mathcal{S}\times\mathcal{A}\rightarrow\mathbb{R}^d$. Furthermore, the transition kernel and reward function are assumed to be linear in $\phi$. The LSVI-UCB algorithm is shown in Algorithm~\ref{alg1}. For lines 3-6, the agent executes the policy to collect data in an episode. For lines 7-11, the parameter $w_t$ of $Q$-function is updated in closed-form by following the regularized least-squares problem as $w_t\leftarrow \argmin_{w\in\mathbb{R}^d}\sum\nolimits_{\tau=0}^{m}\bigl[r_t(s_t^\tau,a_t^\tau)+\max_{a\in\mathcal{A}}Q_{t+1}(s_{t+1}^{\tau},a)-w^{\top}\phi(s_t^{\tau},a_t^{\tau})\bigr]^2+\lambda\|w\|^2$,
where $m$ is the total number of episodes, and $\tau$ is the episodic index. The least-squares problem has the explicit solution $w_t=\Lambda_t^{-1}\sum_{\tau=0}^{m}\phi(x_t^{\tau},a_t^{\tau})\bigl[r_t(x_t^{\tau},a_t^{\tau})+\max_a Q_{t+1}(x_{t+1}^{\tau},a)\bigr]$ (line 9), where $\Lambda_t$ is the Gram matrix. The value function is estimated by $Q_t(s,a)\approx w_t^{\top}\phi(s,a)$. 
LSVI-UCB uses an UCB-bonus \citep{abbasi-2011} in line 10
\begin{equation}\label{eq:lsvi-bonus} 
r^{\rm ucb}=[\phi(s,a)^\top\Lambda_t^{-1}\phi(s,a)]^{\nicefrac{1}{2}}
\end{equation}
to measure the uncertainty of state-action pairs. The term $u:=(\phi^\top\Lambda_t^{-1}\phi)^{-1}$ can be intuitively considered as a pseudo count of the state-action pair in the representation space of $\phi$. Thus, the bonus $r^{\rm ucb}=\nicefrac{1}{\sqrt{u}}$ represents the uncertainty along the direction of $\phi$. By adding the bonus to the $Q$-value, we obtain an optimistic value function $Q^+$, which serves as an upper bound of $Q$ to encourage exploration. The bonus in each step is propagated from the end of the episode by the backward update of the $Q$-value (lines~7-11), which follows the principle of dynamic programming. Theoretical analysis shows that LSVI-UCB achieves a near-optimal worst-case regret of $\tilde{\mathcal{O}}(\sqrt{d^3 T^3 L^3})$ with proper selection of $\alpha$ and $\lambda$, where $L$ is the total number of steps.

LSVI-UCB~\citep{jin-2019} has been demonstrated to be effective in principled exploration. Nevertheless, developing a practical exploration algorithm for DRL is challenging, since (\romannumeral 1) the UCB-bonus utilized by LSVI-UCB is specifically defined for linear MDPs, and (\romannumeral 2) LSVI-UCB utilizes backward update of $Q$-functions (lines 7-11 in Alg.~\ref{alg1}) to aggregate uncertainty. Although the backward update is a standard approach in theoretical analysis of sample-efficient exploration \citep{shani2020optimistic,cai2020provably,wang2019optimism}, such an approach is scarcely studied in developing practical exploration algorithm for DRL.

\section{Proposed Method}

OB2I solves the efficient exploration problem for DRL in the following directions:
\begin{itemize}[leftmargin=1.0em,itemsep=0.1em,topsep=0.1em]
\item we propose a general-purpose UCB-bonus for optimistic exploration. More specifically, we utilize bootstrapped DQN to construct a general-purpose UCB-bonus, which is theoretically consistent with LSVI-UCB for linear MDPs. We refer to \S~\ref{sec-ucb-bonus} for the details;
\item we integrate bootstrapped $Q$-functions and UCB-bonus into the backward update, which follows the principle of dynamic programming. More specifically, we extend Episodic Backward Update (EBU)~\citep{ebu-2019} from standard $Q$-learning to bootstrapped $Q$-learning, and we refer this extension to as Bootstrapped EBU (BEBU). We refer to \S~\ref{sec-epi-update} for the details.
\end{itemize}

\subsection{General-Purpose UCB-Bonus}\label{sec-ucb-bonus}

\begin{figure*}[t]
\centering
\subfigure[]{\includegraphics[width=1.8in]{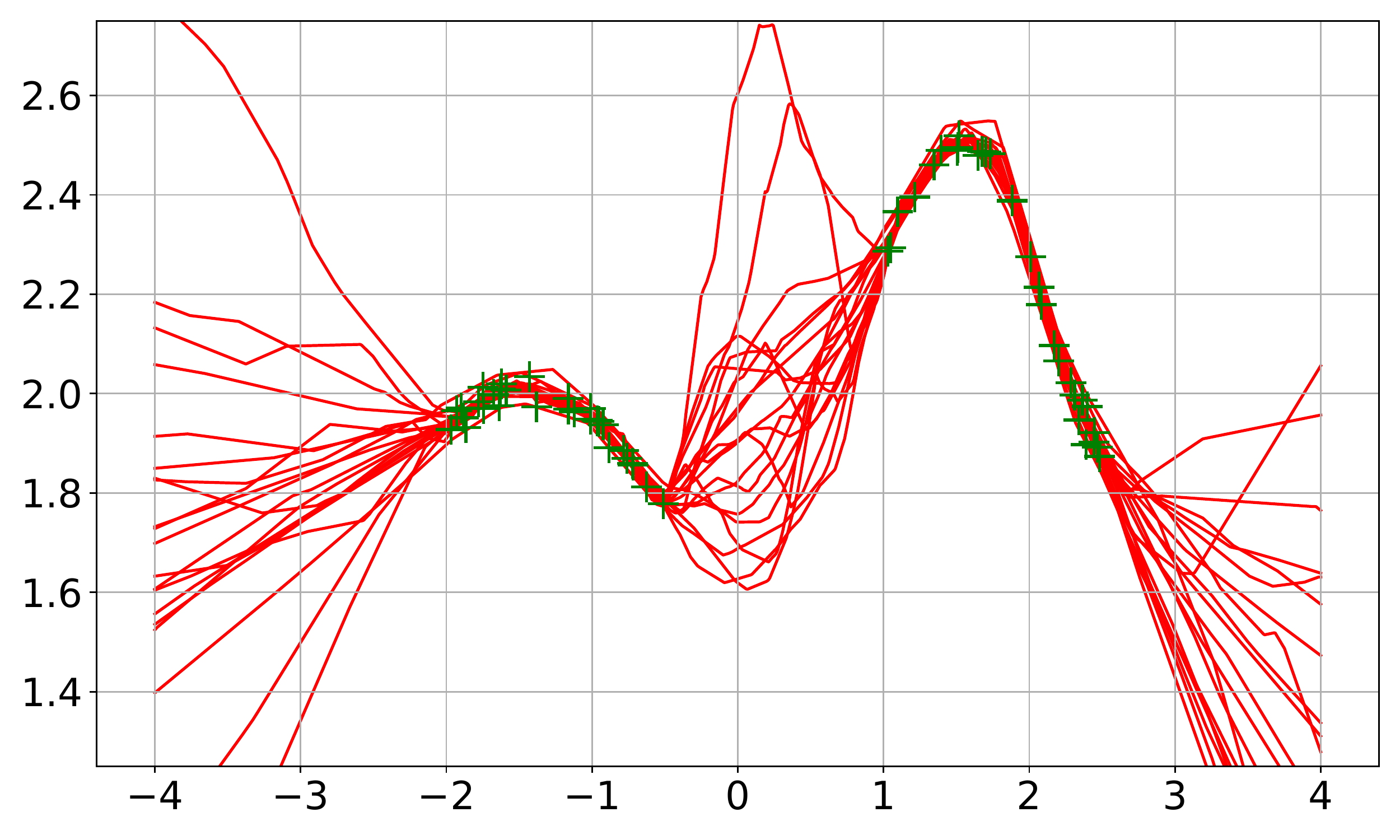}\label{fig:ucb-a}}\hspace{2em}
\subfigure[]{\includegraphics[width=1.8in]{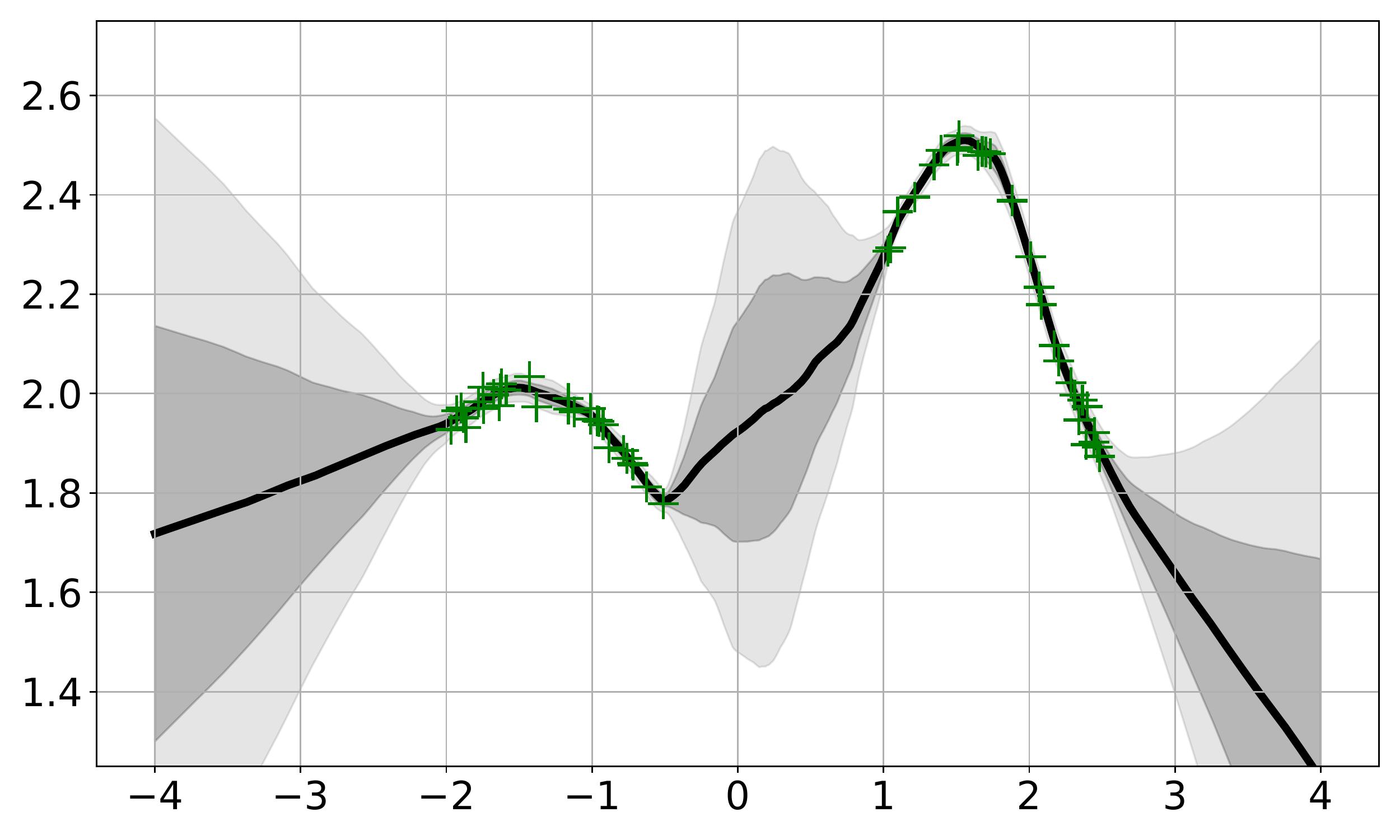}\label{fig:ucb-b}}\hspace{2em}
\subfigure[]{\includegraphics[width=1.8in]{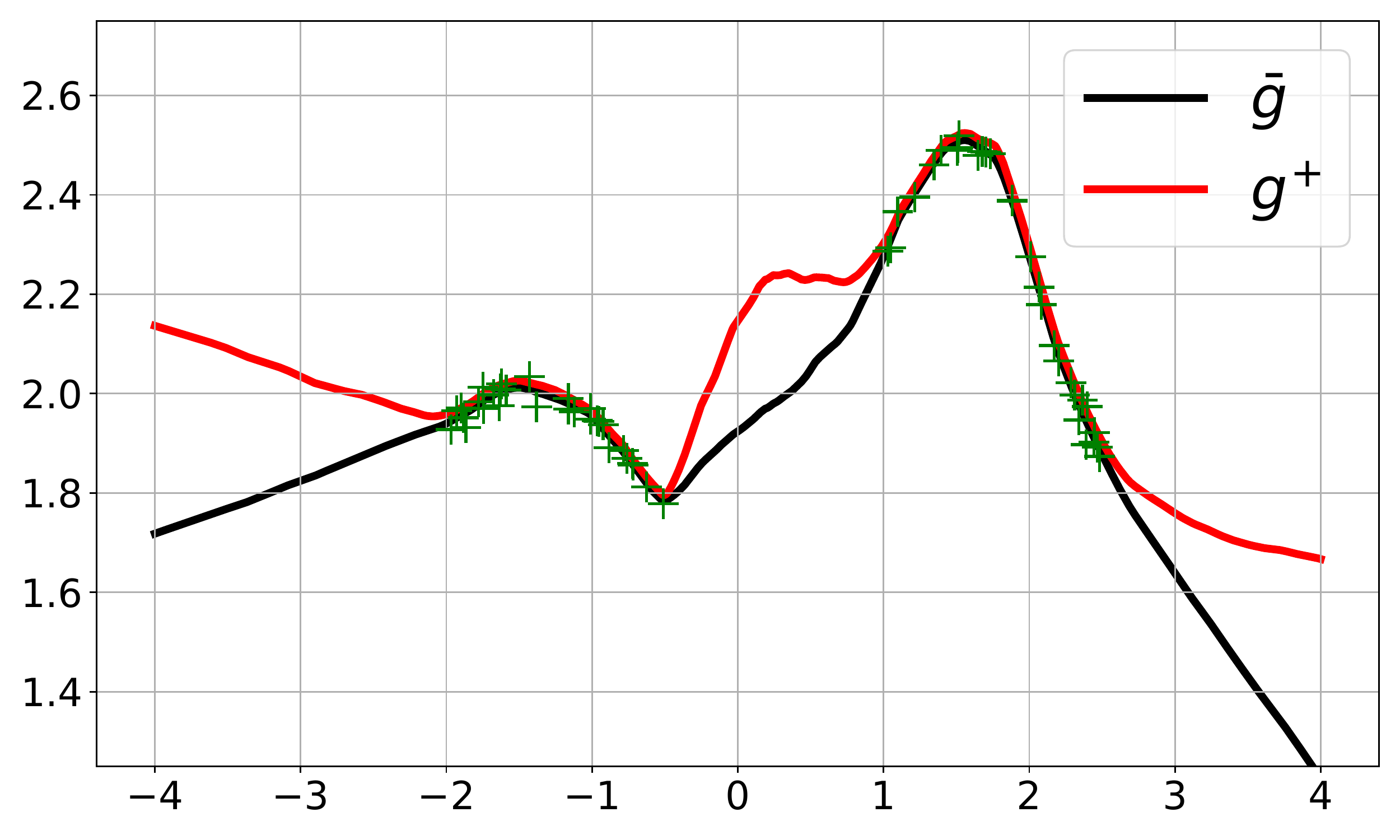}\label{fig:ucb-c}}
\caption{Illustration of the general-purpose UCB-bonus in a simple regression task. Green markers indicate there are 60 data points. (a) Regression curves of $20$ neural networks. (b) Mean estimation (black curve) and uncertainty measurement (shadow region). (c) The optimistic value (red) and mean value (black).}
\label{fig:ucb}
\end{figure*}

Optimistic exploration uses an optimistic action-value function $Q^{+}$ to encourage exploration by adding a bonus term to the standard $Q$-value. Thus $Q^+$ serves as an upper bound of the standard $Q$. The bonus term represents the epistemic uncertainty that results from lacking experiences of the corresponding states and actions. For DRL with deep $Q$ network, it is impractical to derive a closed-form optimistic bonus like~\eqref{eq:lsvi-bonus}. Instead, we propose a general-purpose UCB-bonus $\mathcal{B}(s_t,a_t)$ by measuring the disagreement of multiple bootstrapped $Q$-values $\{Q^k(s_t,a_t)\}_{k=1}^{K}$ of the state-action pair $(s_t,a_t)$ in a bootstrapped DQN. That is, 
\begin{equation}\label{eq:bonus}
\mathcal{B}(s_t,a_t):=\sqrt{\frac{1}{K}\sum_{k=1}^{K}\Bigl(Q^k(s_t,a_t)-\bar{Q}(s_t,a_t)\Bigr)^2},
\end{equation}
where $\bar{Q}(s_t,a_t)$ is the mean of the bootstrapped $Q$-values. A similar uncertainty measurement was used in \citet{ucb-2017}. We discuss the difference between \citet{ucb-2017} and our algorithm in \S\ref{sec:related-work}. We surprisingly find that this simple form in \eqref{eq:bonus} is also provably efficient for linear MDPs. Indeed, the following theorem establishes the connection between the general-purpose UCB-bonus defined in \eqref{eq:bonus} and the bonus in LSVI-UCB defined in \eqref{eq:lsvi-bonus}.

\begin{thm}
\label{thm::var_informal}
In linear MDPs, the UCB-bonus $\mathcal{B}(s_t,a_t)$ in OB2I is equivalent to the bonus-term $[\phi_t^\top\Lambda_t^{-1}\phi_t]^{\nicefrac{1}{2}}$ in LSVI-UCB, where $\Lambda_t\leftarrow\sum_{\tau=0}^{m}\phi(x_t^{\tau},a_t^{\tau})\phi(x_t^{\tau},a_t^{\tau})^\top+\lambda \cdot \mathrm{\mathbf{I}}$, and $m$ is the current episode. 
\end{thm}

In Theorem \ref{thm::var_informal}, we cast the variance that defines the UCB-bonus of OB2I as the posterior variance of value functions under the Bayesian learning regime. We remark that the bootstrapped distribution of value functions coincides with the posterior under a Bayesian setting where the prior is uninformative \citep{friedman2001elements}. We refer to Appendix~\ref{app-proof} for the details and complete statement. Theorem~\ref{thm::var_informal} shows that the general-purpose UCB-bonus in \eqref{eq:bonus} is provably efficient and equivalent to bonus-term in LSVI-UCB for linear cases. Importantly, \eqref{eq:bonus} is a general form for arbitrary $Q$ functions such as deep neural networks.


Overall, for general DRL problem, using the UCB-bonus $\mathcal{B}(s_t,a_t)$ in \eqref{eq:bonus} is desirable for the following reasons.
\begin{itemize}[leftmargin=1.0em,itemsep=0.1em,topsep=0.1em]
\item Bootstrapped DQN is a non-parametric posterior sampling method, that is naturally compatible with deep neural networks \citep{osband-2019}.
\item $\mathcal{B}(s_t,a_t)$ quantifies the epistemic uncertainty of $(s_t,a_t)$. Due to the non-convexity nature of optimizing neural network and independency of random initialization, if $(s_t, a_t)$ is scarcely visited, $\mathcal{B}(s_t,a_t)$ obtained via bootstrapped $Q$-values will tend to be large. Moreover, $\mathcal{B}(s_t,a_t)$ converges to zero asymptotically as the samples increases to infinity. 
\item $\mathcal{B}(s_t,a_t)$ is computed for batch data sampled from experience replay. This is more efficient than other optimistic methods that change the action-selection scheme in each timestep~\citep{ucb-2017,info-2019} to choose optimistic actions based on uncertainty estimation or information-directed sampling.
\end{itemize}
The optimistic $Q^{+}$ is obtained by summing up $\mathcal{B}(s_t,a_t)$ and the estimated $Q$-function, which takes the form as
\begin{equation}
Q^{+}(s_t,a_t):=Q(s_t,a_t)+\alpha\mathcal{B}(s_t,a_t),
\end{equation}
where $\alpha$ is a tuning parameter. We use a simple regression task with neural networks to illustrate the proposed UCB-bonus, as shown in Figure~\ref{fig:ucb}. We use 20 neural networks with the same network architecture to solve the same regression problem. 
According to~\citet{bootstrap-2016}, the differences among the outcomes of fitting the 20 neural networks is a result of random initializations. For a given input $x$, the networks yield different estimations $\{g_i(x)\}_{i=1}^{20}$. It follows from Figure~\ref{fig:ucb-a} that the estimations $\{g_i(x)\}_{i = 1}^{20}$ behave similar in the region with large amount of observations, resulting in small disagreement of the estimations. However, for regions with less observations, the disagreement of the estimations inflates a lot. In Figure~\ref{fig:ucb-b}, we illustrate the confidence bound of the regression results $\bar{g}(x)\pm \tilde{\sigma}(g_i(x))$ and $\bar{g}(x)\pm 2\tilde{\sigma}(g_i(x))$, where $\bar{g}(x)$ and $\tilde{\sigma}(g_i(x))$ are the mean and standard deviation of the estimations. The standard deviation $\tilde{\sigma}(g_i(x))$ captures the epistemic uncertainty of regression results. Figure~\ref{fig:ucb-c} shows the optimistic estimation $g^{+}(x)=\bar{g}(x)+\tilde{\sigma}(g_i(x))$ plus the standard deviation. Clearly, the optimistic estimation $g^{+}$ is close to $\bar{g}$ in the region with dense observations, and it is larger than $\bar{g}$ in the region with fewer observations. 

In DRL, the bootstrapped $Q$-functions $\{Q^k(s_t,a_t)\}_{k=1}^{K}$, estimated by fitting the target $Q$-function, perform similarly as $\{g_i(x)\}_{i=1}^{20}$ in the above regression task. A higher UCB-bonus $\mathcal{B}(s_t,a_t):=\tilde{\sigma}(Q^k(s_t,a_t))$ indicates a higher epistemic uncertainty of the action-value function with $(s_t,a_t)$. Therefore, $Q^{+}$ produces optimistic estimation for novel state-action pairs and behaves similar to the $Q$-function in areas that are well explored by the agent. Hence, the optimistic estimation $Q^{+}$ encourages the agent to explore the potentially informative state-action pairs efficiently.

\subsection{Backward Induction of Uncertainty}\label{sec-epi-update}

OB2I adopts BEBU for backward induction when updating the action-value function. BEBU collects a complete trajectory from the replay buffer for each update. Such an approach allows OB2I to infer the long-term effect in an episode for decision making. In contrast, DQN and Bootstrapped DQN sample one-step transitions, which loses the information containing long-term effects.

It has to be mentioned that BEBU is required to propagate future uncertainty to the estimated action-value function consistently via UCB-bonus. For instance, let $t_2>t_1$ be indices of two steps in an episode. If $Q_{t_2}$ updates after that of $Q_{t_1}$, then the uncertainty propagated to $Q_{t_1}$ is inconsistent with that propagated to $Q_{t_2}$. 

To integrate the general-purpose UCB-bonus into bootstrapped $Q$-learning, we propose a novel $Q$-target by adding the bonus in both the immediate reward and the next-$Q$ value. The proposed $Q$-target needs to be suitable for BEBU in training. Formally, the $Q$-target for updating $Q^k$ is defined as 
\begin{equation}\label{eq:update-ucb}
\begin{split}
\!y&_t^k:=\bigl[r(s_t,a_t)+\alpha_1\mathcal{B}(s_t,a_t;\theta)\bigr]+\gamma\\
&\bigl[Q^k(s_{t+1},a';\theta^{k-})+\alpha_2\mathds{1}_{a'\neq a_{t+1}}\mathcal{\tilde{B}}^k (s_{t+1},a';\theta^{-})\bigr],
\end{split}
\end{equation}
where $a'=\argmax_a Q^k(s_{t+1},a;\theta^{k-})$. 
The choice of $a'$ is determined by the target $Q$-value without considering the bonus. The immediate reward is added by $\mathcal{B}(s_t,a_t;\theta)$ with a factor $\alpha_1$, where the bonus $\mathcal{B}$ is computed by bootstrapped $Q$-network with parameter $\theta$. The next-$Q$ value is added by $\mathds{1}_{a'\neq a_{t+1}} \mathcal{\tilde{B}}^k (s_{t+1},a';\theta^{-})$ with factor $\alpha_2$, where the bonus $\mathcal{\tilde{B}}^k$ is computed by the target network with parameter $\theta^{-}$. We assign different bonus $\mathcal{\tilde{B}}^k$ of next-$Q$ value to different heads, since the choices of $a'$ are different among the heads. Meanwhile, we assign the same bonus $\mathcal{B}$ of immediate reward for all the heads. We introduce an indicator function $\mathds{1}_{a'\neq a_{t+1}}$ to control backward update of $Q$-values. More specifically, in the $t$-th step, the action-value function $Q^k$ is updated optimistically at the state-action pair $(s_{t+1}, a_{t+1})$ due to the backward update. Thus, we ignore the bonus of next-$Q$ value in the update of $Q^k$ when $a'$ is equal to $a_{t+1}$.

We use an example to illustrate the process of backward update. We store and sample the episodic experiences in a replay buffer. Considering an episode containing three time steps, $(s_0,a_0)\rightarrow (s_1,a_1)\rightarrow (s_2,a_2)$. We thus update the $Q$-value in the head $k$ in the backward manner, namely $Q(s_2,a_2)\rightarrow Q(s_1,a_1)\rightarrow Q(s_0,a_0)$ from the end of the episode. We describe the process as follows,
\begin{itemize}[leftmargin=1.0em,itemsep=0.1em,topsep=0.1em]
\item first, we update $Q(s_2,a_2)\leftarrow r(s_2,a_2)+\alpha_1\mathcal{B}(s_2,a_2)$. Note that in the last time step, we do not need to consider the next-$Q$ value;
\item then, we have $Q(s_1,a_1)\leftarrow [r(s_1,a_1)+\alpha_1\mathcal{B}(s_1,a_1)]+
[Q(s_2,a')+\alpha_2\mathds{1}_{a'\neq a_2}\mathcal{\tilde{B}}(s_2,a')]$ by following \eqref{eq:update-ucb}, where $a'=\argmax_a Q(s_2,a)$. Since $Q(s_2,a_2)$ is updated optimistically in the first step, we ignore the bonus-term $\mathcal{\tilde{B}}$ in next-$Q$ value when $a'=a_2$. The UCB-bonus is augmented by adding $\mathcal{B}$ and $\mathcal{\tilde{B}}$ to the immediate reward and next-$Q$ value, respectively;
\item as for $Q(s_0,a_0)$, its update follows the same principle. The optimistic $Q$-value is $Q(s_0,a_0)\leftarrow [r(s_0,a_0)+\alpha_1\mathcal{B}(s_0,a_0)]+[Q(s_1,a')+\alpha_2\mathds{1}_{a'\neq a_1}\mathcal{\tilde{B}}(s_1,a')]$, where $a'=\argmax_a Q(s_1,a)$.
\end{itemize}

In practice, the episodic update typically leads to instability in DRL due to strong correlation in consecutive transitions. Hence, we propose a diffusion factor $\beta\in[0,1]$ in BEBU to prevent such instability as that used in ~\citet{ebu-2019}. The $Q$-value is therefore computed as the weighted sum of the current value and the back-propagated estimation scaled with factor $\beta$. We consider an episodic experience that contains $T$ transitions, denoted by $E=\{\mathbf{S},\mathbf{A},\mathbf{R},\mathbf{S'}\}$, where $\mathbf{S}=\{s_0,\ldots,s_{T-1}\}$, $\mathbf{A}=\{a_0,\ldots,a_{T-1}\}$, $\mathbf{R}=\{r_0,\ldots,r_{T-1}\}$ and $\mathbf{S}'=\{s_1,\ldots,s_T\}$. We initialize a $Q$-table $\mathbf{\tilde{Q}}\in\mathbb{R}^{K\times|\mathcal{A}|\times T}$ by $Q(\cdot;\theta^{-})$ to store the next-$Q$ values of all the next states $\mathbf{S}'$ and valid actions for $K$ heads. We initialize $\mathbf{y}\in\mathbb{R}^{K\times T}$ to store the $Q$-target for $K$ heads and $T$ steps. We use bootstrapped $Q$-network with parameters $\theta$ to compute the bonus $\mathbf{B}=[\mathcal{B}(s_0,a_0),\ldots,\mathcal{B}(s_{T-1},a_{T-1})]$ for immediate reward, and use the target network with parameters $\theta^{-}$ to compute bonus $\mathbf{\tilde{B}}^k=[\mathcal{\tilde{B}}^k(s_1,a'_1),\ldots,\mathcal{\tilde{B}}^k(s_{T},a'_{T})]$ for next-$Q$ value in each head, where $a'_t=\argmax_a Q^k(s_t,a;\theta^{k-})$. The bonus vector $\mathbf{B}\in\mathbb{R}^{T}$ is the same for all $Q$-heads, while $\mathbf{\tilde{B}}\in\mathbb{R}^{K\times T}$ contains different values for different heads because the choices of $a'_t$ are different.

In the training of head $k$, we initialize the $Q$-target in the last step by $\mathbf{y}[k,T-1]=\mathbf{R}_{T-1}+\alpha_1\mathbf{B}_{T-1}$. We then perform a recursive backward update to get all $Q$-target values. The elements of $\mathbf{\tilde{Q}}[k,a_{t+1},t]$ for step $t$ in head $k$ is updated by using its corresponding $Q$-target $\mathbf{y}[k,t+1]$ with the diffusion factor as follows, 
\begin{equation}
\mathbf{\tilde{Q}}[k,a_{t+1},t]\leftarrow \beta\mathbf{y}[k,t+1]+(1-\beta)\mathbf{\tilde{Q}}[k,a_{t+1},t].
\end{equation}
Then, we update $\mathbf{y}[k,t]$ in the previous time step based on the newly updated $t$-th column of $\mathbf{\tilde{Q}}[k]$ as follows,
\begin{equation}
\mathbf{y}[k,t]\leftarrow \bigl(\mathbf{R}_t+\alpha_1\mathbf{B}_t\bigr)+\gamma\bigl(\mathbf{\tilde{Q}}[k,a',t]+\alpha_2\mathds{1}_{a'\neq a_{t+1}}\mathbf{\tilde{B}}[k,t]\bigr),
\end{equation}
where $a'=\argmax_{a} \mathbf{\tilde{Q}}[k,a,t]$. In practice, we construct a matrix $\mathbf{\tilde{A}}=\argmax_{a} \mathbf{\tilde{Q}}[\cdot,a,\cdot]\in\mathbb{R}^{K\times T}$ to gather all the actions $a'$ that correspond to the next-$Q$, and then construct a mask matrix $\mathbf{M}\in \mathbb{R}^{K\times T}$ to store the information whether $\mathbf{\tilde{A}}$ is identical to the executed action in the corresponding timestep or not. The bonus of next-$Q$ is the element-wise product of $\mathbf{M}$ and $\mathbf{\tilde{B}}$ with factor $\alpha_2$. After the backward update, we compute the $Q$-value of $(\mathbf{S},\mathbf{A})$ as $\mathbf{Q}=Q(\mathbf{S},\mathbf{A};\theta)\in\mathbb{R}^{K\times T}$. The loss function takes the form of $L(\theta)=\mathbb{E}\bigl[(\mathbf{y}-\mathbf{Q})^2|(s_t,a_t,r_t,s_{t+1})\in E,~E\sim\mathcal{D}\bigr]$, where the episodic experience $E$ is sampled from replay buffer to perform gradient descent. The gradients of all heads can be computed simultaneously via BEBU. We refer the full algorithm of OB2I to Appendix~\ref{app-sec-alg}.

To summarize, we use BEBU to propagate the future uncertainty in an episode, which is an extension of EBU~\citep{ebu-2019}. Compared to EBU, BEBU requires extra tensors to store the UCB-bonus for immediate reward and next-$Q$ value, which are integrated to propagate uncertainties. Meanwhile, integrating uncertainty into BEBU needs special design by using the mask. The previous works \citep{ucb-2017,sunrise-2020} do not propagate the future uncertainty and, therefore, does not capture the core benefit of utilizing UCB-bonus for the exploration of MDPs. We highlight that OB2I propagates future uncertainty in a time-consistent manner based on BEBU, which exploits the theoretical analysis established by~\citet{jin-2019}. Only in this way, $Q^+$ incorporates the epistemic uncertainty across {\it multiple steps}, so that the greedy action with respect to $Q^+$ (at the decision stage) performs deep exploration. In contrast, separating the bonus function from the bootstrapping process (i.e., only using it at the decision stage) fails to propagate uncertainty. The backward update also empirically improves the sample-efficiency significantly by allowing bonuses and delayed rewards to propagate through transitions of a complete episode.

\subsection{Comparison with LSVI-UCB}


We remark that both LSVI-UCB and OB2I constructs the confidence interval of value functions based on the frequentist approaches. Specifically, LSVI-UCB constructs the confidence intervals explicitly based on the linear model, whereas OB2I constructs the confidence interval based on the non-parametric bootstrapped approach. In OB2I, we adopt Bootstrapped $Q$-values to calculate the standard deviation of $Q$-functions with neural network parameterization, which coincides with the bonus in LSVI-UCB on linear MDPs. When the sample size increases, the distribution of bootstrapped $Q$-values converges asymptotically to the posterior under a Bayesian setting where the prior is uninformative \citep{friedman2001elements}. Hence, in Theorem~\ref{thm::var_informal}, we use the Bayesian setting as a simplification to motivate our algorithm while this is not necessary. A recent approach also uses a similar way to motivate the worst-case regret of randomized value functions \citep{russo-2019}.

From an empirical perspective, LSVI-UCB requires strict linear assumption in the transition dynamics and value function. To the opposite, OB2I uses a non-parametric form and the general UCB-bonus works for arbitrary $Q$-function types such as deep neural networks. In OB2I, the neural networks can be updated by gradient descent using batch episodic trajectories sampled from the replay buffer in each training step. However, in LSVI-UCB, all historical samples have to be used to update the $Q$-function and calculate the confidence bonus in each training step, since the posterior matrix $\Lambda$ relies on the update-to-date representation $\phi$ which is varies as the training proceeds. As a consequence, OB2I is much more sample-efficient empirically. Moreover, in LSVI-UCB, the target $Q$-function is updated in each iteration, whereas in OB2I, the target-network is updated less frequent. Similar empirical tricks are commonly used in most existing off-policy DRL algorithms \citep{bootstrap-2016,ddpg-2015,td3-2018}.

\section{Related Work}\label{sec:related-work}
We discuss a number of closely related approaches in this section and choose the most important ones to compare in our experiment.
One practical principle for exploration in DRL is maintaining the epistemic uncertainty.
Epistemic uncertainty comes from the unawareness of the environments, and it decreases as the exploration proceeds.
Bootstrapped DQN~\citep{bootstrap-2016,bootstrap-2018} samples $Q$-values from the randomized value functions to encourage exploration through Thompson sampling. \citet{ucb-2017} proposes to use the standard-deviation of bootstrapped Q-functions to measure the uncertainty. Although the uncertainty measurement is similar to that of OB2I, our method is different from \citet{ucb-2017} in the following aspects: (i) our approach propagates the uncertainty through backward update; (ii) \citet{ucb-2017} does not use the bonus in the update of $Q$-functions and their bonus is computed when taking the actions; (iii) we establish theoretical connections between the proposed UCB-bonus and LSVI-UCB. SUNRISE~\citep{sunrise-2020} extends \citet{ucb-2017} to continuous control through confidence reward and weighted Bellman backup. Information-Directed Sampling (IDS)~\citep{info-2019} is based on bootstrapped DQN, and chooses actions by balancing the instantaneous regret and information gain. OAC~\citep{oac-2019} uses two $Q$-networks to get lower and upper bounds of the $Q$-value to perform exploration in continuous control tasks. These methods seek to estimate the epistemic uncertainty and choose the optimistic actions. In contrast, we use the uncertainty of value function to construct intrinsic rewards and perform backward update, which propagates future uncertainty to the estimated $Q$-value. 

Uncertainty Bellman Equation (UBE) \citep{uncer-2018} proposes an upper bound on the variance of the posterior of $Q$-values, which is further utilized for optimism in exploration. 
Bayesian-DQN~\citep{bayesian-dqn} replaces the last layer in deep $Q$-network with Bayesian Linear Regression (BLR) that estimates a posterior of the $Q$-function. These methods use parametric distributions to describe the posterior while OB2I uses the bootstrapped method to construct the confidence bonus. UBE and BLR also require inverting a large matrix in training and hence is computational expensive. Previous methods also utilize the epistemic uncertainty of dynamics through Bayesian posterior~\citep{IGM-2020} and ensembles~\citep{disagree-2019}. Nevertheless, they consider single-step uncertainty, while we consider the long-term uncertainty in an episode.

To measure the novelty of states for constructing count-based intrinsic rewards, previous methods have attempted to use density model~\citep{count-2016,count-2017}, static hashing~\citep{ex2-2017,Contingency-2019,optpi-2020}, episodic curiosity~\citep{reach-2019,ngu-2020}, curiosity-bottleneck~\citep{cb-2019}, information gain~\citep{vime-2016} and prediction error from random networks~\citep{RND-2019} for novelty evaluation. The curiosity-driven exploration based on prediction-error of environment models such as ICM~\citep{curiosity-2017,largescale-2019}, EMI~\citep{emi-2019}, and variational dynamics~\citep{bai-2020} enable the agents to explore in a self-supervised manner. According to~\citet{bonux-2020}, although bonus-based methods show promising results in hard exploration tasks like Montezuma's Revenge, they do not perform well on other Atari games. Meanwhile, NoisyNet~\citep{noise2-2018} performs significantly better than bonus-based methods evaluated by the entire Atari suite. Overall, \citet{bonux-2020} suggests that the pace of the exploration progress might have been obfuscated by some promising results only on a few selected hard exploration games. We follow this principle and evaluate OB2I on the Atari suite with 49 games.

Beyond model-free methods, model-based RL also uses optimism for planning and exploration~\citep{nix1994estimating}. Model-assisted RL \citep{kalweit2017uncertainty} uses ensembles to make use of artificial data with high uncertainty. \citet{buckman2018sample} uses ensemble dynamics and $Q$-functions to use model rollouts when they do not cause large errors. Planning to explore~\citep{sekar2020planning} seeks out future uncertainty by integrating uncertainty to Dreamer~\citep{dreamer-2020}. Ready Policy One~\citep{ball2020ready} optimizes policies for both reward and model uncertainty reduction. Noise-Augmented RL~\citep{pacchiano2020optimism} uses statistical bootstrap to generalize the optimistic posterior sampling~\citep{agrawal2017posterior} to DRL. Hallucinated UCRL~\citep{curi2020efficient} reduces optimistic exploration to exploitation by enlarging the control space. The model-based RL needs to estimate the posterior of dynamics, while OB2I relies on the posterior of $Q$-functions. 

\section{Experimental Results}

\subsection{Environmental Baselines}

We evaluate the algorithms in high-dimensional image-based tasks, including MNIST Maze \citep{ebu-2019} and 49 Atari games. We refer Appendix~\ref{app:minst-maze} for the experiments on MNIST Maze, and discuss the experiments on Atari games in this section. Directly comparing OB2I with baselines using Bootstrapped DQN is not fair, since OB2I uses backward update for training. To achieve fair comparison, we reimplement all Bootstrapped DQN-based baselines with BEBU. We compare the following methods in experiments:
\begin{itemize}[leftmargin=1.0em,leftmargin=1.0em,itemsep=0.1em,topsep=0.1em]
\item \textbf{OB2I}: the proposed principled exploration method.
\item \textbf{BEBU}: a reimplementation of Bootstrapped DQN~\citep{bootstrap-2016} with BEBU.
\item \textbf{BEBU-UCB}: BEBU with optimistic actions selected by the upper bound of $Q$~\citep{ucb-2017,sunrise-2020}. 
\item \textbf{BEBU-IDS}: integrating homoscedastic IDS~\citep{info-2019} into BEBU without distributional RL. 
\end{itemize}
We refer to Appendix~\ref{app-sec-alg} for the algorithmic comparison between all methods. According to EBU~\citep{ebu-2019}, the backward update is significantly more sample-efficient than standard $Q$-learning by using only 20M training frames to achieve the mean human-normalized score of standard DQN, which requires 200M training frames. We follow this setting by training all BEBU-based methods with 20M frames. In our experiments, 20M frames in OB2I is sufficient to produce strong empirical results and achieve competitive results with several baselines using 200M frames. 

We additionally compare the performance of DQN~\citep{DQN-2015}, NoisyNet~\citep{noise2-2018}, Bootstrapped DQN (BootDQN)~\citep{bootstrap-2016}, BootDQN-IDS~\citep{info-2019}, UBE~\citep{uncer-2018} in 200M training frames, and Bayesian DQN~\citep{bayesian-dqn} in 20M training frames. We choose NoisyNet as a baseline since it has been evaluated on the entire Atari suite (instead of several hard exploration games) such that it performs substantially better than existing bonus-based methods~\citep{bonux-2020}, including CTS-counts~\citep{count-2016}, PixelCNN-counts \citep{count-2017}, RND~\citep{RND-2019}, and ICM~\citep{curiosity-2017}. UBE and Bayesian-DQN are selected as baselines because they use parametric functions to approximate the posterior of $Q$-values, while OB2I uses a non-parametric bootstrap. BootDQN-IDS has been demonstrated to be a strong baseline~\citep{info-2019} based information-directed sampling and BootDQN.

\subsection{Evaluation Metric and Hyperparameters}

An ensemble policy by a majority vote of $Q$-heads is used for 30 no-op  evaluation. The no-op evaluation indicates a setting that 30 no-op actions are first executed in each evaluation episode to provide diversity for the agent~\citep{DQN-2015}. The majority-vote combines all the heads into a single ensemble policy, which follows the same evaluation method as in~\citet{bootstrap-2016}. We use the popular human-normalized score $\frac{\rm Score_{Agent}-Score_{Random}}{\rm |Score_{human}-Score_{random}|}$ as a baseline score. In Atari games, \citet{bootstrap-2016} observes that the bootstrapping does not contribute much in performance. Empirically, Bootstrapped DQN uses the same samples to train all $Q$-heads in each training step. This empirical simplification is also adopted by \citet{ucb-2017,bootstrap-2018,info-2019}. We use such a simplification for OB2I and all bootstrapped-based methods.

For OB2I, we set both $\alpha_1$ and $\alpha_2$ as $0.5\times 10^{-4}$ by tuning over five popular tasks, including Breakout, Freeway, Qbert, Seaquest, and SpaceInvaders. Generally, small $\alpha_1$ and $\alpha_2$ yield better performance empirically since the bonus accumulates along the episode that usually contains thousands of steps in Atari. We use diffusion factor $\beta=0.5$ for all methods by following~\citet{ebu-2019}. We refer to Appendix~\ref{app-sec-implement} for the detailed specifications. The code is available at \url{https://github.com/Baichenjia/OB2I}.

\begin{figure}[t]
\centering
\includegraphics[width=3.3in]{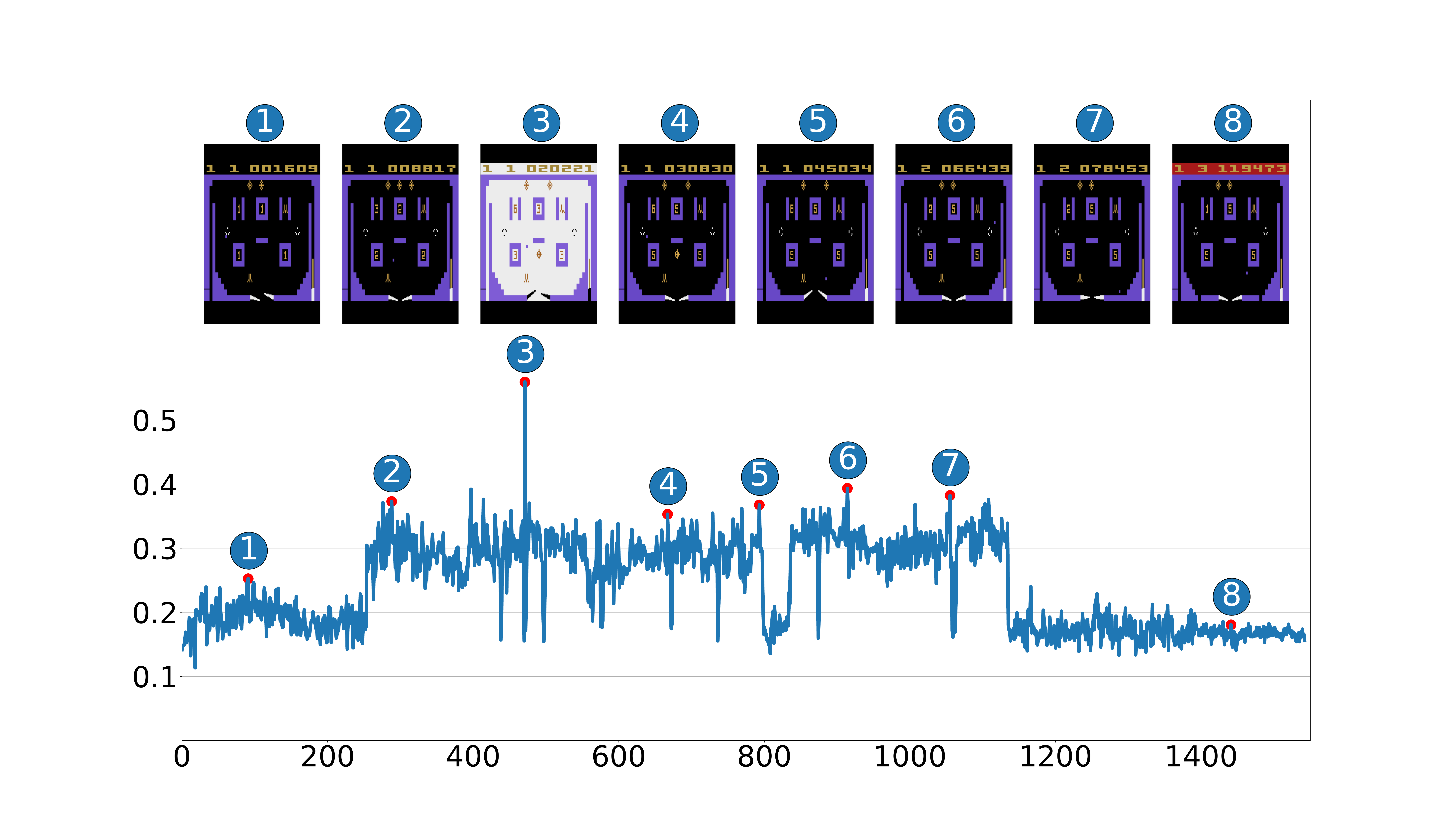}
\caption{Visualizing UCB-bonus in VideoPinball. See video at \url{https://rb.gy/xmzw4g}.}
\label{fig:vis-pinball-small}
\end{figure}

\begin{table*}[t]
\caption{Summary of human-normalized scores in 49 Atari games. BEBU, BEBU-UCB, BEBU-IDS and OB2I are trained for 20M frames with RTX-2080Ti GPU for 5 random seeds.}
\vspace{0.5em}
  \label{tab:scores-atari-sum}
  \centering
\setlength{\tabcolsep}{0.9mm}{
\begin{tabular}{c|ccccc|ccccc}
    \hline
    {Frames} & \multicolumn{5}{|c|}{200M} & \multicolumn{5}{|c}{20M} \cr
	\hline
    {~} & DQN & UBE & BootDQN & NoisyNet & \textbf{BootDQN-IDS} & Bayesian-DQN & BEBU & BEBU-UCB & \textbf{BEBU-IDS} & \textbf{OB2I} \cr
    \hline
    {Mean} & 241\% & 440\% & 553\% & 651\% & \textbf{757\%} & 224\% & 553\% & 610\% & \textbf{622\%} & \textbf{765\%} \cr
    {Median} & 93\% & 126\% & 139\% & 172\% & \textbf{187\%} & 27\% & 36\% & 38\% & \textbf{44\%} & \textbf{50\%} \cr
\hline
\end{tabular}}
\end{table*}

\subsection{Main Results and Visualization}

Table~\ref{tab:scores-atari-sum} reports the overall performance of all the methods on 49 Atari games. According to Table~\ref{tab:scores-atari-sum}, BootDQN-IDS performs better than UBE, BootDQN, and NoisyNet. Thus, BootDQN-IDS outperforms popular bonus-based exploration methods that perform worse than NoisyNet \citep{bonux-2020}. We then reimplement BootDQN-IDS with BEBU, and we refer this version to as BEBU-IDS. We observe that OB2I outperforms BEBU-IDS in both mean and medium scores, as well as outperforming all other bonus-based methods in the backward update setting. We report the detailed raw scores in Appendix~\ref{app-raw-score}. Moreover, Appendix~\ref{app-raw-score-comp} shows that OB2I outperforms BEBU, BEBU-UCB, and BEBU-IDS in 36, 34, and 35 games out of all 49 games, respectively. 

To understand the general-purpose UCB-bonus, we use a trained OB2I agent to interact with the environment for an episode in VideoPinball and record the UCB-bonuses at each step. OB2I improves the performance of VideoPinball significantly and achieves the best score among all baselines. In this task, the pinball moves fast in the playfield to hit bumpers, spinners and rollovers to score points. Our UCB-bonus estimates the uncertainty of interacting with different objects to encourage the pinball to hit less frequently visited objects. The curve in Figure~\ref{fig:vis-pinball-small} shows the UCB-bonuses of the subsampled steps in the episode. We choose eight spikes and visualize the corresponding frames. The events in spikes correspond to rarely hit objects or crucial events, which are important for the agent to obtain rewards: hitting the rollover (1,4,6), using flippers to send the pinball back into the playfield when it drops to the bottom (2), hitting the specific lit target (3), hitting the bumpers and spinners (7,8), and losing the ball (5). Most obviously, the UCB-bonus increases significantly at spike 3 because the ball hit a specific lit target that causes the screen to flash and the agent scores 1000 points, while hitting other objects gets less than 100 points. In the last stage (including spike 8), the UCB-bonuses are low since the score has reached the upper limit and the flippers are locked. We provide more visualization examples in Appendix~\ref{app-vis-OB2I}. 

\begin{figure}[t]
\centering
\includegraphics[width=2.7in]{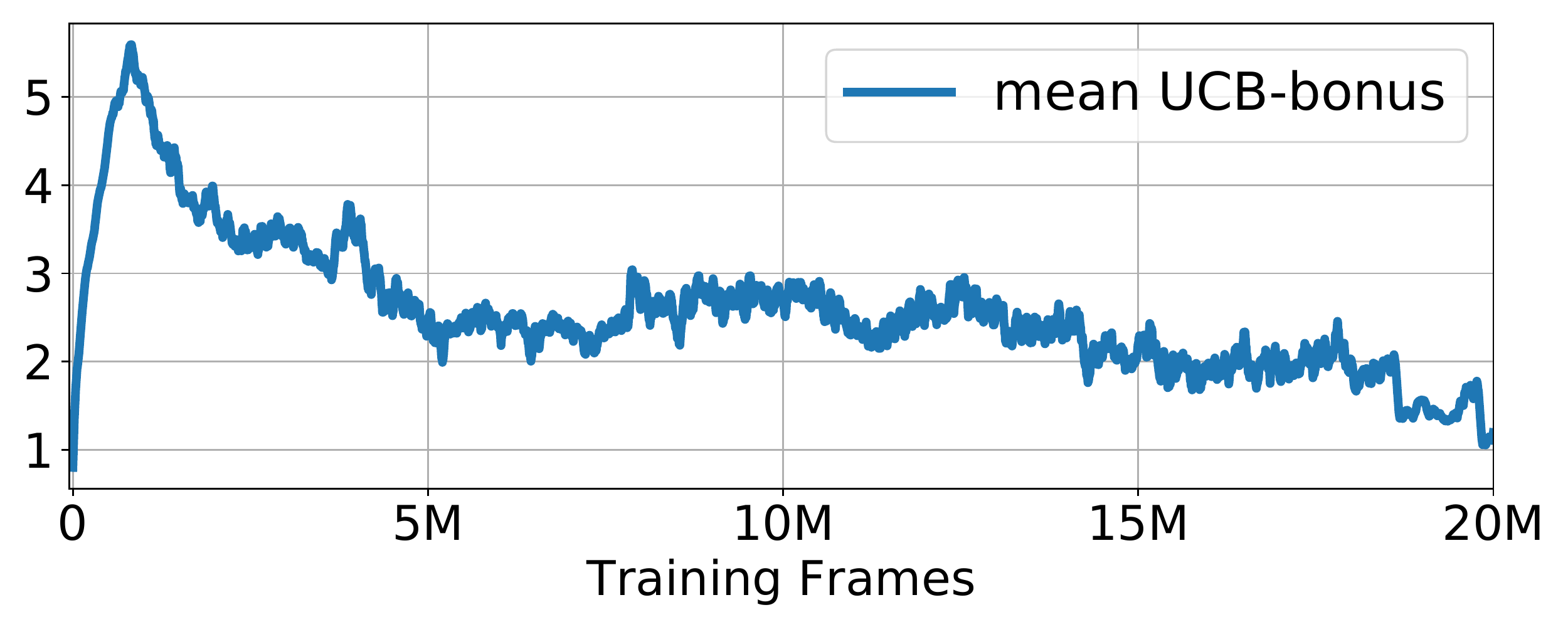}
\caption{The change of mean UCB-bonus in the learning process.}
\label{fig:bonus-total}
\end{figure}

We further record the the mean of the UCB-bonus of the training batch in the learning process. The result is shown in Figure~\ref{fig:bonus-total}. The UCB-bonus is low at the beginning since the networks are randomly initialized. When the agent starts to explore the environment, the mean UCB-bonus increases rapidly to award exploration. As more experiences of state-action pairs are gathered, the mean UCB-bonus reduces gradually, indicating that the bootstrapped value functions concentrate around the optimal value and the epistemic uncertainty decreases. Nevertheless, according to Figure~\ref{fig:vis-pinball-small}, the UCB-bonuses are relatively high at scarcely visited areas or crucial events, and therefore the bonuses promote exploration for the corresponding events.

\subsection{Ablation Study}

We conduct an ablation study to better understand the importance of backward update and bonus term in OB2I. The results of the ablation studies are provided in Table~\ref{tab:ablation}. We observe that (\romannumeral 1) when we use the ordinary update strategy by sampling transitions instead of episodes, OB2I reduces to BootDQN-UCB with significant performance loss. This is consistent with previous conclusions in~\citep{ebu-2019} that backward update is crucial for sample-efficient training; (\romannumeral 2) when the UCB-bonus is set to 0, OB2I reduces to BEBU; (\romannumeral 3) when both the backward update and UCB-bonus are removed, OB2I reduces to standard BootDQN, which performs poorly in 20M training frames; (\romannumeral 4) to illustrate the effect of the proposed UCB-bonus, we substitute it with the popular RND-bonus~\citep{RND-2019}. Specifically, we use an independent RND network to generate RND-bonus for each state in training. The RND-bonus is added to both the immediate reward and next-$Q$. The result shows that our proposed UCB-bonus outperforms RND-bonus without introducing additional complexities compared to BootDQN.

\begin{table}[t]
\small
\caption{Ablation Study}
\vspace{0.2em}
\label{tab:ablation}
\centering
\setlength{\tabcolsep}{0.6mm}{
\begin{tabular}{l|cc|ccc}
    \hline
    {~}         & Backward        & Bonus   & Qbert & SpaceInvaders & Freeway \cr
    \hline
    OB2I        & \checkmark      & UCB     & 4275.0    & 904.9 & 32.1  \cr
    BootDQN-UCB &     -           & UCB     & 3284.7    & 731.8 & 20.5  \cr
    BEBU        & \checkmark      & -       & 3588.4    & 814.4 & 21.5  \cr
    BootDQN     &    -            & -       & 2206.8    & 649.5 & 18.3  \cr
    \hline
    BEBU-RND    & \checkmark      & RND     & 3702.5    & 832.7 & 22.6  \cr
\hline
\end{tabular}}
\end{table}

\section{Conclusion}

In this work, we have proposed a principled exploration method, i.e., OB2I, that shares nice theoretical properties as LSVI-UCB. By integrating with backward induction, the sample efficiency is further enhanced. We evaluate OB2I empirically by solving MNIST maze and 49 Atari games. Results show that OB2I outperforms several strong baselines. The visualizations suggest that high UCB-bonus corresponds to informative experiences for exploration. As far as we see, our work seems to establish the first empirical attempt of uncertainty propagation in deep RL, which exploits the core benefit of theoretical analysis. Moreover, we observe that the connection between theoretical analysis and practical algorithm provides strong empirical performance, which hopefully raises insights on combining theory and practice to the community. Future directions include adapting OB2I to continuous control and integrating OB2I with other expressive bonus schemes.

\section*{Acknowledgements}

This work was supported in part by the National Natural Science Foundation of China under Grant 51935005, in part by the Fundamental Research Program under Grant JCKY20200603C010, and in part by the Science and Technology on Space Intelligent Laboratory under Grant ZDSYS-2018-02. The authors thank Tencent Robotics X for the computation resources supported. The authors also thank the anonymous reviewers, whose invaluable comments and suggestions have helped us to improve the paper.

\nocite{langley00}

\bibliography{OB2I-full}
\bibliographystyle{icml2021}

\onecolumn
\appendix

\icmltitle{Principled Exploration via Optimistic Bootstrapping and Backward Induction \\ (Appendix)}

\section{UCB Bonus in OB2I}\label{app-proof}
Recall that we consider the following regularized least-square problem,
\#\label{eq::regression_problem}
w_t\leftarrow\argmin_{w\in\mathbb{R}^d}\sum_{\tau=0}^{m}\bigl[r_t(s_t^\tau,a_t^\tau)+\max_{a\in\mathcal{A}}Q_{t+1}(s_{t+1}^{\tau},a)-w^{\top}\phi(s_t^{\tau},a_t^{\tau})\bigr]^2 + \lambda \|w\|^2.
\#
In the sequel, we consider a Bayesian linear regression perspective of (\ref{eq::regression_problem}) that captures the intuition behind the UCB-bonus in OB2I. Our objective is to approximate the action-value function $Q_t$ via fitting the parameter $w$, such that
\$
w^\top\phi(s_t, a_t) \approx r_t(s_t,a_t)+\max_{a\in\mathcal{A}}Q_{t+1}(s_{t+1},a),
\$
where $Q_{t+1}$ is given. We assume that we are given a Gaussian prior of the initial parameter $w \sim \mathcal N(0, \mathrm{\mathbf{I}}/\lambda)$. With a slight abuse of notation, we denote by $w_t$ the Bayesian posterior of the parameter $w$ given the set of independent observations $\cD_m = \{(s^\tau_t, a^\tau_t, s^\tau_{t+1})\}_{\tau \in [0,m]}$. 
We further define the following noise with respect to the least-square problem in (\ref{eq::regression_problem}),
\#\label{eq::noise}
\epsilon = r_t(s_t, a_t) + \max_{a\in\cA} Q_{t+1}(s_{t+1}, a) - w^\top\phi(s_t, a_t),
\#
where $(s_t, a_t, s_{t+1})$ follows the distribution of trajectory. The following theorem justifies the UCB-bonus in OB2I under the Bayesian linear regression perspective.
\begin{thm}[Formal Version of Theorem \ref{thm::var_informal}]
\label{thm::var}
We assume that $\epsilon$ follows the standard Gaussian distribution $\mathcal N(0, 1)$ given the  state-action pair $(s_t, a_t)$ and the parameter $w$. Let $w$ follows the Gaussian prior $\cN(0, \mathrm{\mathbf{I}}/\lambda)$. We define
\#\label{eq::def_lambda}
\Lambda_t=\sum_{\tau=0}^{m}\phi(x_t^{\tau},a_t^{\tau})\phi(x_t^{\tau},a_t^{\tau})^\top+\lambda \cdot \mathrm{\mathbf{I}}.
\#
It then holds for the posterior of $w_t$ given the set of independent observations $\cD_m = \{(s^\tau_t, a^\tau_t, s^\tau_{t+1})\}_{\tau \in [0,m]}$ that
\$
\text{\rm Var}\bigl(\phi(s_t, a_t)^\top w_t \bigr) =  \text{\rm Var}\bigl( \tilde Q_t(s_t, a_t)\bigr) = \phi(s_t, a_t)^\top \Lambda^{-1}_t \phi(s_t, a_t), \quad \forall (s_t, a_t)\in\cS\times\cA.
\$
Here we denote by $\tilde Q_t = w_t^\top \phi$ the estimated action-value function.
\end{thm}
\begin{proof}
The proof follows the standard analysis of Bayesian linear regression. See, e.g., \citet{west1984outlier} for a detailed analysis. We denote the target of the linear regression in (\ref{eq::regression_problem}) by 
\$
y_t = r_t(s_t, a_t) + \max_{a\in\cA} Q_{t+1}(s_{t+1}, a).
\$
By the assumption that $\epsilon$ follows the standard Gaussian distribution, we obtain that 
\#\label{eq::pf_density_y}
y_t \given (s_t, a_t), w \sim \mathcal{N}\bigl(w^\top \phi(s_t, a_t), 1\bigr).
\#
Recall that we have the prior distribution $w \sim \mathcal N(0, \mathrm{\mathbf{I}}/\lambda)$. Our objective is to compute the posterior density $w_t = w \given \cD_m$, where $\cD_m = \{(s^\tau_t, a^\tau_t, s^\tau_{t+1})\}_{\tau \in [0,m]}$ is the set of observations. It holds from Bayes rule that 
\#\label{eq::pf_bayes_rule}
\log p(w \given \cD_m) = \log p(w) + \log p(\cD_m\given w) + Const.,
\#
where $p(\cdot)$ denote the probability density function of the respective distributions. Plugging (\ref{eq::pf_density_y}) and the probability density function of Gaussian distribution into (\ref{eq::pf_bayes_rule}) yields
\#\label{eq::pf_density_posterior}
\log p(w \given \cD_m) &= -\|w\|^2/2 -\sum^m_{\tau = 1} \| w^\top\phi(s^\tau_t, a^\tau_t) - y^\tau_t\|^2/2 + Const.\notag\\
&=-(w - \mu_t)^\top \Lambda^{-1}_t(w - \mu_t)/2 + Const.,
\#
where we define 
\$
\mu_t = \Lambda^{-1}_t \sum^m_{\tau = 1}\phi(s^\tau_t, a^\tau_t) y^\tau_t, \qquad \Lambda_t=\sum_{\tau=0}^{m}\phi(x_t^{\tau},a_t^{\tau})\phi(x_t^{\tau},a_t^{\tau})^\top+\lambda \cdot \mathrm{\mathbf{I}}.
\$ 
Thus, by (\ref{eq::pf_density_posterior}), we obtain that $w_t = w\given \cD_m \sim \mathcal N(\mu_t, \Lambda^{-1}_t)$. It then holds for all $(s_t, a_t)\in\cS\times\cA$ that
\$
\text{\rm Var}\bigl(\phi(s_t, a_t)^\top w_t \bigr) =  \text{\rm Var}\bigl( \tilde Q_t(s_t, a_t)\bigr) = \phi(s_t, a_t)^\top \Lambda^{-1}_t \phi(s_t, a_t), 
\$
which concludes the proof of Theorem \ref{thm::var}.
\end{proof}
\begin{remark}[Extension to Neural Network Parameterization]
{\rm 
We remark that our proof can be extended to explain deep neural network parametrization under the overparameterized network regime \citep{arora2019fine}. Under such a setting, a two-layer neural network $f(\cdot; W)$ with parameter $W$ and ReLU activation function can be approximated by
\$
f(x; W) &\approx f(x; W_0) + \phi_{W_0}(x)^\top (W - W_0) = \phi_{W_0}(x)^\top W, \quad \forall x\in\mathcal{X},
\$
where the approximation holds if the neural network is sufficiently wide \cite{arora2019fine}. Here $W_0$ is the Gaussian distributed initial parameter and $\phi_{W_0} = ([\phi_{W_0}]_1, \ldots, [\phi_{W_0}]_m)^\top$ is the feature embedding defined as follows,
\$
[\phi_{W_0}(x)]_r = \frac{1}{\sqrt{m}}\sigma\bigl( x^\top [W_0]_r \bigr), \quad \forall x\in\mathcal{X},~r\in[m].
\$
Hence, if we consider a Bayesian perspective of training neural network, where the parameter $W$ is obtained by solving a Bayesian linear regression with the feature $\phi_{W_0}$, then the proof of Theorem \ref{thm::var} can be applied to the setting upon conditioning on the random initialization $W_0$. Thus, Theorem \ref{thm::var} applies to the neural network parameterization under such an overparameterized neural network regime.
}
\end{remark}
\clearpage
\section{Algorithmic Description}\label{app-sec-alg}

\begin{algorithm}[h!]
\caption{OB2I in DRL}
\label{alg2}
\begin{algorithmic}[1]
\STATE {{\bf Initialize:} replay buffer $\mathcal{D}$, bootstrapped $Q$-network $Q(\cdot;\theta)$ and target network $Q(\cdot;\theta^{-})$}
\STATE {{\bf Initialize:} total training frames $H=20{\rm M}$, current frame $h=0$}
\WHILE {$h<H$}
\STATE {Pick a bootstrapped $Q$-function to act by sampling $k\sim \rm{Unif}\{1,\ldots,K\}$}
\STATE {Reset the environment and receive the initial state $s_0$}
\FOR {step $i=0$ {\bfseries to} Terminal}
\STATE {With $\epsilon$-greedy choose \textcolor{blue}{$a_i=\argmax_a Q^k(s_i,a)$}}
\STATE {Take action and observe $r_i$ and $s_{i+1}$, then save the transition in buffer $\mathcal{D}$}
\IF    {$h~\% ~{\rm training~frequency} = 0$}
\STATE {Sample an episodic experience $E=\{\mathbf{S},\mathbf{A},\mathbf{R},\mathbf{S}'\}$ with length $T$ from $\mathcal{D}$}
\STATE {Initialize a $Q$-table $\tilde{\mathbf{Q}}=Q(\mathbf{S}',\mathcal{A};\theta^{-})\in\mathbb{R}^{K\times |\mathcal{A}|\times T}$ by the target $Q$-network}
\STATE {Compute the UCB-bonus for immediate reward for all steps to construct $\mathbf{B}\in\mathbb{R}^{T}$}
\STATE {Compute the action matrix $\mathbf{\tilde{A}}=\argmax_{a} \mathbf{\tilde{Q}}[\cdot,a,\cdot]\in\mathbb{R}^{K\times T}$ to gather all $a'$ of next-$Q$}
\STATE {Compute the UCB-bonus for next-$Q$ for all heads and all steps to construct $\mathbf{\tilde{B}}\in\mathbb{R}^{K\times T}$}
\STATE {Compute the mask matrix $\mathbf{M}\in \mathbb{R}^{K\times T}$ where $\mathbf{M}[k,t]=\mathds{1}_{\mathbf{\tilde{A}}[k,t]\neq \mathbf{A}_{t+1}}$}
\STATE {Initialize target table $\mathbf{y}\in\mathbb{R}^{K\times T}$ to zeros, and set $\mathbf{y}[\cdot,T-1]=\mathbf{R}_{T-1}+\alpha_1\mathbf{B}_{T-1}$}
\FOR {$t=T-2$ {\bfseries to} $0$}
\STATE {\textcolor{blue}{$\tilde{\mathbf{Q}}[\cdot,a_{t+1},t]\leftarrow \beta \mathbf{y}[\cdot,t+1]+(1-\beta)\tilde{\mathbf{Q}}[\cdot,a_{t+1},t]$}}
\STATE {\textcolor{blue}{$\mathbf{y}[\cdot,t]\leftarrow \bigl(\mathbf{R}_t+\alpha_1\mathbf{B}_t\bigr)+\gamma\bigl(\mathbf{\tilde{Q}}[\cdot,a',t]+\alpha_2\mathbf{M}[\cdot,t]\circ\mathbf{\tilde{B}}[\cdot,t]\bigr)$ where $a'=\mathbf{\tilde{A}}[\cdot,t]$}}
\ENDFOR
\STATE {Compute the $Q$-value of $(\mathbf{S},\mathbf{A})$ for all heads as $\mathbf{Q}=Q(\mathbf{S},\mathbf{A};\theta)\in\mathbb{R}^{K\times T}$}
\STATE {Perform a gradient descent step on $(\mathbf{y}-\mathbf{Q})^2$ with respect to $\theta$}
\ENDIF
\STATE {Every $C$ steps reset $\theta^{-}\leftarrow\theta$}
\STATE {$h\leftarrow h+1$}
\ENDFOR
\ENDWHILE
\end{algorithmic}
\end{algorithm}

\begin{remark}[Remark on $\epsilon$-Greedy]{\rm
We adopt the $\epsilon$-greedy technique based on the empirical concerns. Empirically, $\epsilon$-greedy is helpful at the early stage of training, since the bootstrapped $Q$-heads typically lack diversity at the early stage of training. As shown in Figure \ref{fig:bonus-total}, the bonus for OB2I is small at the begining of training. A similar observation also arises in Randomized Prior Function \citep{bootstrap-2018}, where each head is initialized together with a random but fixed prior function to improves the diversity between Q-heads at the initialization. In OB2I, we use $\epsilon$-greedy as an empirical technique to improve the diversity of Q-heads at the beginning of training while diminishing $\epsilon$-term to zero as the training evolves. For a fair comparison, in our experiments, we preform $\epsilon$-greedy for all BEBU-based baselines (BEBU, BEBU-UCB, and BEBU-IDS) with the same values of $\epsilon$. We remark that the $\epsilon$-greedy technique is also widely used in implementations of methods based on Bootstrapped DQN, including Bootstrapped DQN implementation at \url{https://github.com/johannah/bootstrap_dqn}, \url{https://github.com/rrmenon10/Bootstrapped-DQN}, Sunrise \cite{ucb-2017,sunrise-2020} implementation at \url{https://github.com/pokaxpoka/sunrise}, and the official IDS \citep{info-2019} implementation at \url{https://github.com/nikonikolov/rltf}. NoisyNet \citep{noise2-2018} implementation also applies this technique at \url{https://github.com/Kaixhin/Rainbow}.

In addition, from a theoretical perspective, adopting $\epsilon$-greedy policies in place of greedy policies will hinder the performance difference term $\langle \pi^k, Q^* - Q^{k} \rangle$ in the analysis of LSVI-UCB \citep{jin-2019}, which is upper bounded by zero if $\pi^k$ is the greedy policy corresponding to $Q^k$. In contrast, if $\pi^k$ is the $\epsilon$-greedy policy, adding and subtracting the greedy policy yields an $\epsilon Q_{\max}$ upper bound, which propagates to an additional $O(\epsilon T)$ term in the regret. Therefore, if $\epsilon$ is sufficiently small, the algorithm attains the optimal $\sqrt{T}$-regret. In OB2I, we diminish $\epsilon$-term to zero as the training evolves, which does not incur a large bias to the regret.}
\end{remark}

\clearpage

\begin{algorithm}[h!]
\caption{BEBU \& BEBU-UCB \& BEBU-IDS}
\label{alg-bebu}
\begin{algorithmic}[1]
\STATE {{\bf Input:} Algorithm Type (BEBU, BEBU-UCB, or BEBU-IDS)}
\STATE {{\bf Initialize:} replay buffer $\mathcal{D}$, bootstrapped $Q$-network $Q(\cdot;\theta)$ and target network $Q(\cdot;\theta^{-})$}
\STATE {{\bf Initialize:} total training frames $H=20{\rm M}$, current frame $h=0$}
\WHILE {$h<H$}
\STATE {Pick a bootstrapped $Q$-function to act by sampling $k\sim \rm{Unif}\{1,\ldots,K\}$}
\STATE {Reset the environment and receive the initial state $s_0$}
\FOR {step $i=0$ {\bfseries to} Terminal}
\IF {Algorithm type is \textcolor{blue}{BEBU}}
\STATE {With $\epsilon$-greedy choose \textcolor{blue}{$a_i=\argmax_a Q^k(s_i,a)$}}
\ELSIF {Algorithm type is \textcolor{blue}{BEBU-UCB}}
\STATE {With $\epsilon$-greedy choose \textcolor{blue}{$a_i=\argmax_a [\bar{Q}(s_i,a)+\alpha\cdot\sigma(Q(s_i,a))]$}, where $\bar{Q}(s_i,a_i)=\frac{1}{K}\sum_{k=1}^K Q^k(s_i, a_i)$ and $\sigma(Q(s_i,a_i))=\sqrt{\frac{1}{K}\sum_{k=1}^{K}(Q^k(s_i,a_i)-\bar{Q}(s_i,a_i))^2}$ are the mean and standard deviation of the bootstrapped Q-estimates}
\ELSIF {Algorithm type is \textcolor{blue}{BEBU-IDS}}
\STATE {With $\epsilon$-greedy choose \textcolor{blue}{$a_i=\argmin_{a}\frac{\hat{\Delta}_i(s_i,a)^2}{I_i(s_i,a)}$} by following the regret-information ratio, where $\hat{\Delta}_i(s_i,a_i)=\max_{a'\in\mathcal{A}}u_i(s_i,a')-l_i(s_i,a_i)$ is the expected regret, and $[l_i(s_i,a_i),u_i(s_i,a_i)]$ is the confidence interval. In particular, $u_i(s_i,a_i)=\bar{Q}(s_i,a_i)+\lambda_{\rm ids}\cdot\sigma(Q(s_i,a_i))$ and $l_i(s_i,a_i)=\bar{Q}(s_i,a_i)-\lambda_{\rm ids}\cdot\sigma(Q(s_i,a_i))$. $I(s_i,a_i)=\log (1+\nicefrac{\sigma(Q(s_i,a_i))^2}{\rho^2})+\epsilon_{\rm ids}$ measures the uncertainty, where $\rho$ and $\epsilon_{\rm ids}$ are constants.}
\ELSE 
\STATE {Algorithm type error.}
\ENDIF
\STATE {Take action and observe $r_i$ and $s_{i+1}$, then save the transition in buffer $\mathcal{D}$}
\IF    {$h~\% ~{\rm training~frequency} = 0$}
\STATE {Sample an episodic experience $E=\{\mathbf{S},\mathbf{A},\mathbf{R},\mathbf{S}'\}$ with length $T$ from $\mathcal{D}$}
\STATE {Initialize a $Q$-table $\tilde{\mathbf{Q}}=Q(\mathbf{S}',\mathcal{A};\theta^{-})\in\mathbb{R}^{K\times |\mathcal{A}|\times T}$ by the target $Q$-network}
\STATE {Compute the action matrix $\mathbf{\tilde{A}}=\argmax_{a} \mathbf{\tilde{Q}}[\cdot,a,\cdot]\in\mathbb{R}^{K\times T}$ to gather all $a'$ of next-$Q$}
\STATE {Initialize target table $\mathbf{y}\in\mathbb{R}^{K\times T}$ to zeros, and set $\mathbf{y}[\cdot,T-1]=\mathbf{R}_{T-1}+\alpha_1\mathbf{B}_{T-1}$}
\FOR {$t=T-2$ {\bfseries to} $0$}
\STATE \textcolor{blue}{{$\tilde{\mathbf{Q}}[\cdot,a_{t+1},t]\leftarrow \beta \mathbf{y}[\cdot,t+1]+(1-\beta)\tilde{\mathbf{Q}}[\cdot,a_{t+1},t]$}}
\STATE \textcolor{blue}{{$\mathbf{y}[\cdot,t]\leftarrow \mathbf{R}_t+\gamma \mathbf{\tilde{Q}}[\cdot,a',t]$} where $a'=\mathbf{\tilde{A}}[\cdot,t]$}
\ENDFOR
\STATE {Compute the $Q$-value of $(\mathbf{S},\mathbf{A})$ for all heads as $\mathbf{Q}=Q(\mathbf{S},\mathbf{A};\theta)\in\mathbb{R}^{K\times T}$}
\STATE {Perform a gradient descent step on $(\mathbf{y}-\mathbf{Q})^2$ with respect to $\theta$}
\ENDIF
\STATE {Every $C$ steps reset $\theta^{-}\leftarrow\theta$}
\STATE {$h\leftarrow h+1$}
\ENDFOR
\ENDWHILE
\end{algorithmic}
\end{algorithm}

~\\~\\

\begin{remark}[Remark on Computational Efficiency]
{\rm We remark that OB2I requires much less training time than BEBU-UCB and BEBU-IDS, since both BEBU-UCB and BEBU-IDS requires computing the corresponding confidence bounds in each time step of interaction. In contrast, OB2I only requires estimating the confidence bound for batch training. Meanwhile, the number of interaction steps $L_1$ with the environment are typically set to be much larger than the number of training steps $L_2$ (e.g., in DQN, $L_1\approx4L_2$). Hence, OB2I is more computational efficient under such a conventional setting.}
\end{remark}

\clearpage
\section{Additional Experiment: MNIST Maze}\label{app:minst-maze}

\begin{figure}[h]
\centering
\includegraphics[width=0.45\textwidth]{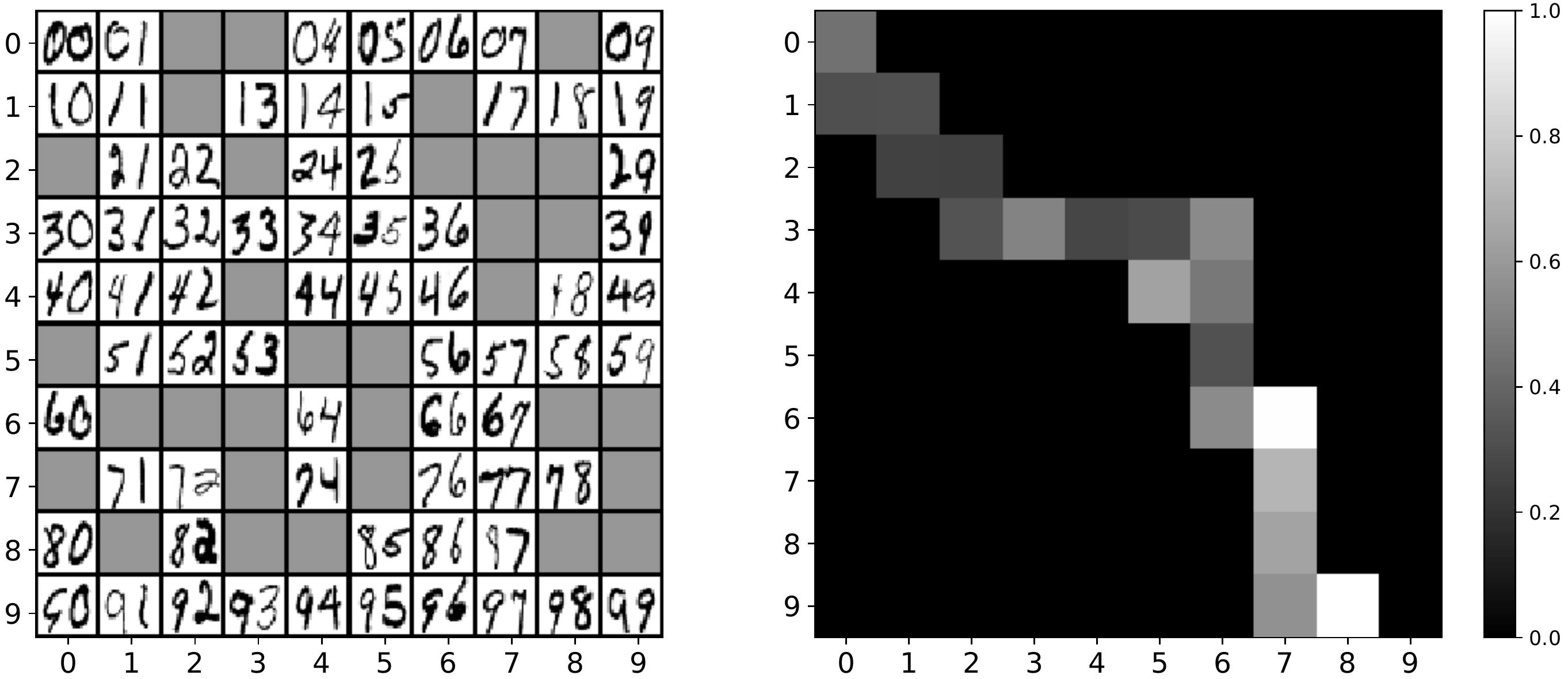}
\caption{An example MNIST maze (left) and the UCB-bonuses in the agent's path (right).}
\label{fig:maze-example}
\end{figure}

\begin{figure}[h]
\centering
\subfigure[Wall density of 30\%]{\includegraphics[width=1.8in]{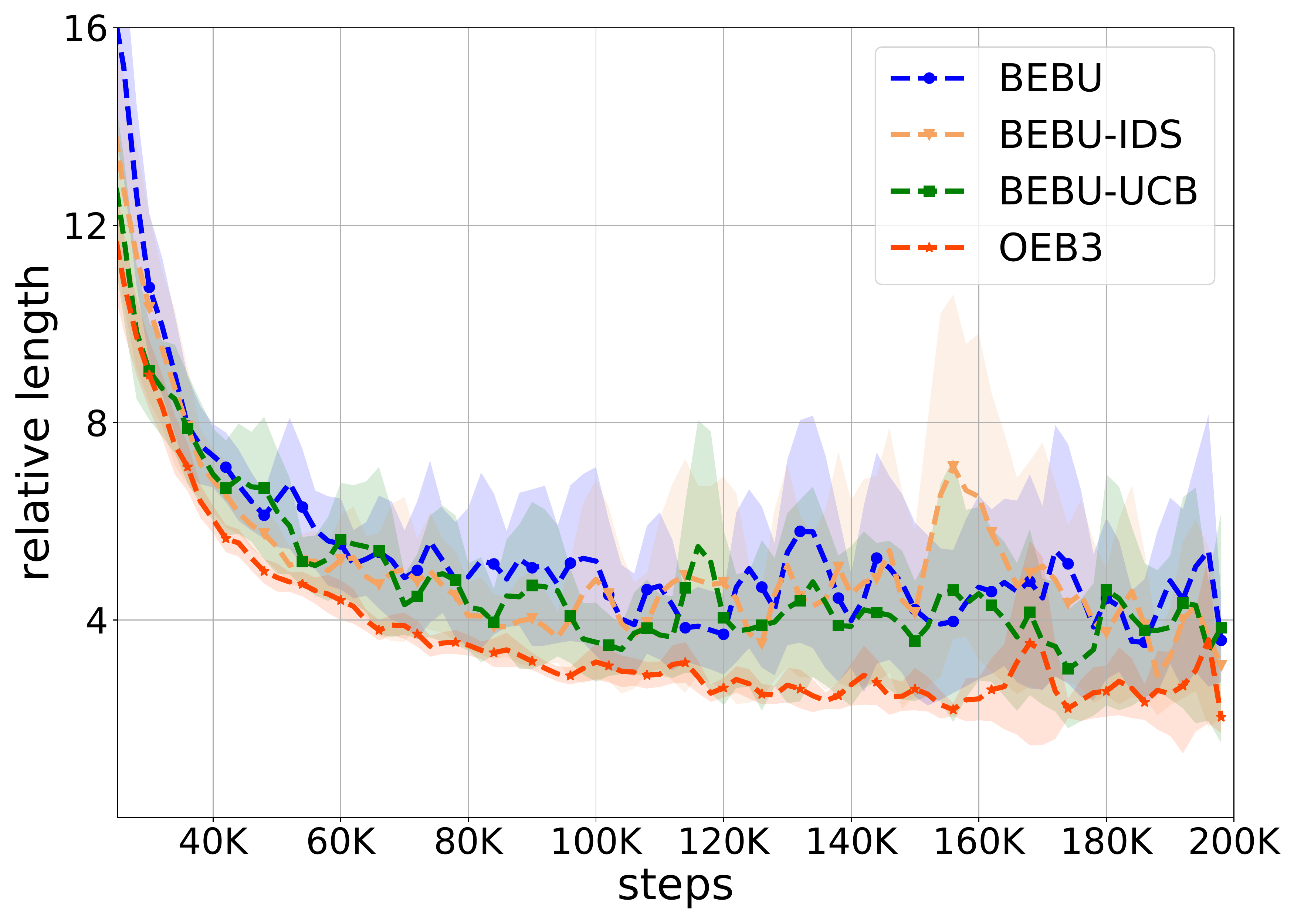}\label{fig:maze-1}}
\subfigure[Wall density of 40\%]{\includegraphics[width=1.8in]{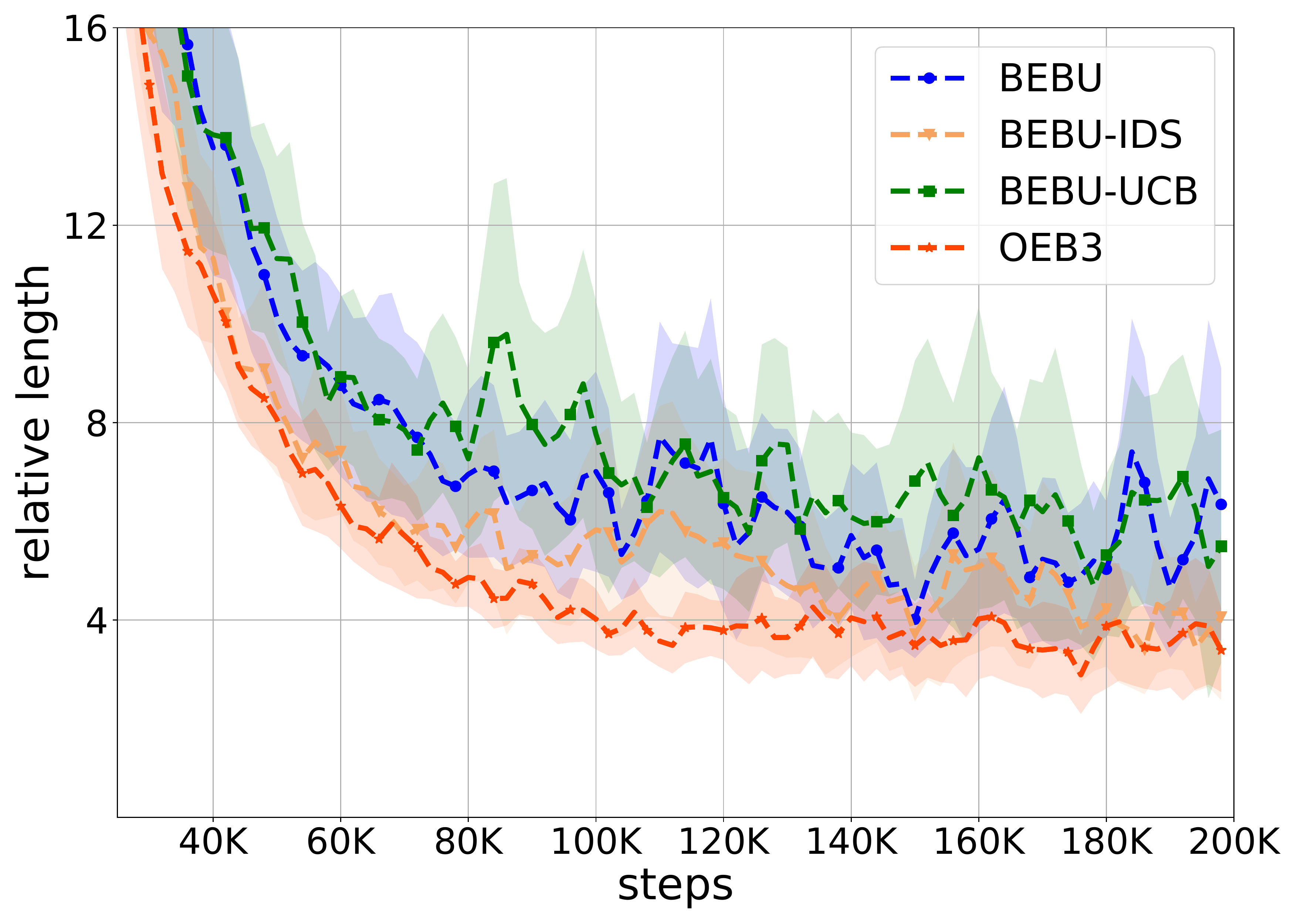}\label{fig:maze-2}}
\subfigure[Wall density of 50\%]{\includegraphics[width=1.8in]{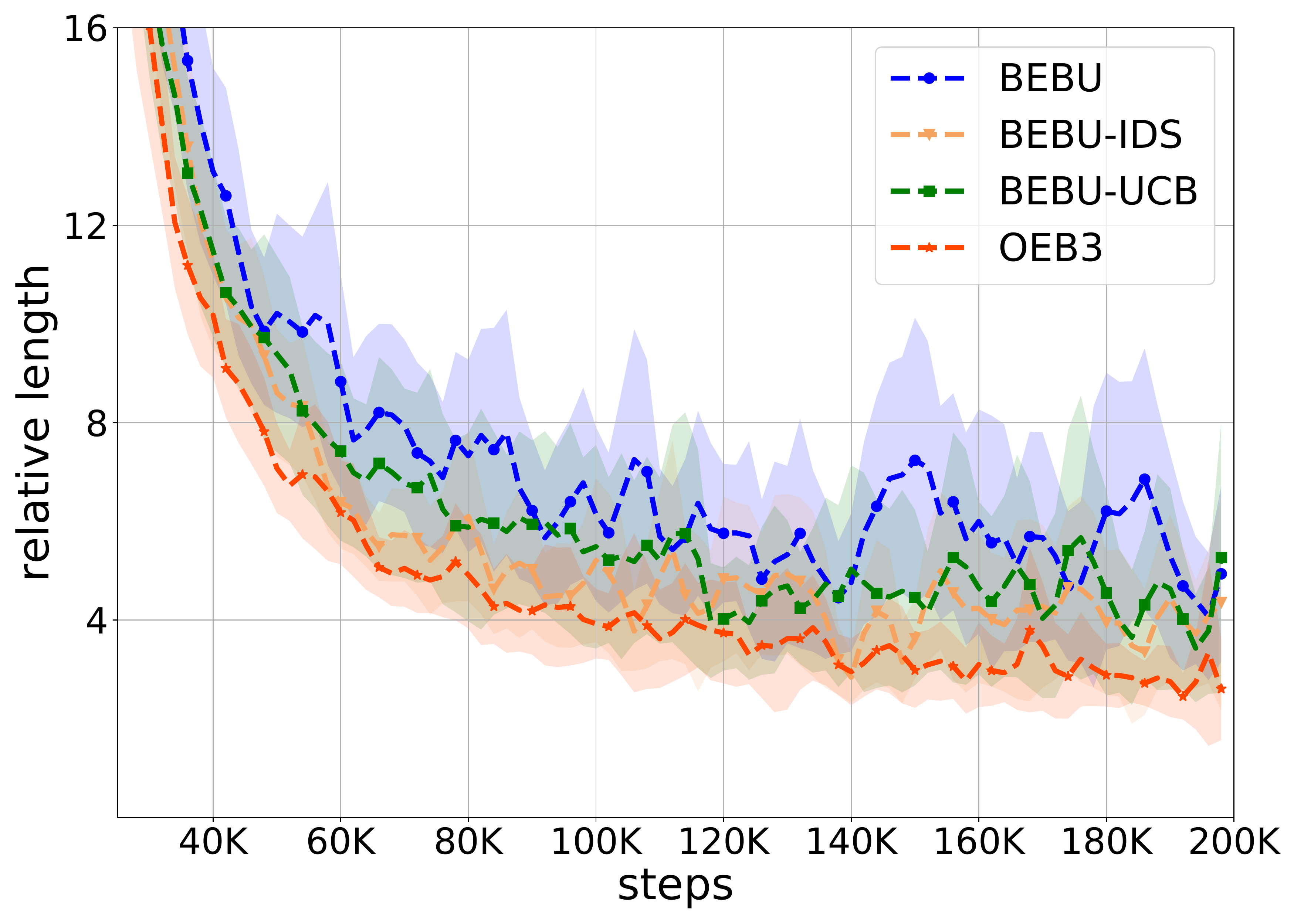}\label{fig:maze-3}}
\caption{Results of 200K steps training of MNIST maze with different wall-density setup.}
\label{fig:maze-result}
\end{figure}

We use $10\times10$ MNIST maze with randomly placed walls to evaluate our method. The agent starts from the initial position $(0,0)$ in the upper-left of the maze and aims to reach the goal position $(9,9)$ in the bottom-right. The state of position $(i,j)$ is represented by stacking two randomly sampled images with label $i$ and $j$ from the MNIST dataset. When the agent steps to a new position, the state representation is reconstructed by sampling images. Hence the agent gets different states even stepping to the same location twice, which minimizes the correlation among locations. Meanwhile, we introduce additional stochasticity in the transition probability. Specifically, the agent has a probability of 10\% to arrive in the adjacent locations when taking an action. For example, when taking action `left', the agent has a 10\% chance of transiting to `up', and a 10\% chance of transiting to `down'. The agent gets a reward of -1 when bumping into a wall, and gets 1000 when reaching the goal. 

We use the different setup of wall-density in the experiment. Here the wall-density means the proportion of walls among all the locations. Figure~\ref{fig:maze-example} (left) shows a generated maze with wall-density of 50\%, where the gray positions represent walls. We train all methods with wall-density of 30\%, 40\%, and 50\%. For each setup, we train 50 independent agents for 50 randomly generated mazes. We use the relative length defined by $l_{\rm agent}/l_{\rm best}$ to evaluate the performance of algorithms, where $l_{\rm agent}$ is the length of the agent's travel to reach the goal in an episode (maximum steps are 1000), and $l_{\rm best}$ is the length of the shortest path to reach the goal. The performance comparison is shown in Figure~\ref{fig:maze-result}. We observe that OB2I performs the best among all the methods. In addition, BEBU-IDS also performs well. To further illustrate the performance of OB2I, We use a trained OB2I agent to take action in the maze presented in Figure~\ref{fig:maze-example} (left). We present the corresponding UCB-bonuses of state-action pairs along the agent's visitation trajectory in Figure~\ref{fig:maze-example} (right). 

We observe that OB2I assigns high UCB-bonus to positions that are critical to exploration. For example, the state-action pairs in location $(3,3)$ and $(6, 7)$ are assigned high UCB-bonus as they are the bottleneck positions in the maze illustrated in Figure~\ref{fig:maze-example} (left), where the agent must visit $(3,3)$ and $(6,7)$ to reach the goal at $(9,9)$. The UCB-bonus encourages the agent to walk through these bottleneck positions correctly. We refer to Appendix~\ref{app-vis-OB2I} for additional examples.

\clearpage
\section{Implementation Detail}\label{app-sec-implement}

\subsection{MNIST Maze}\label{app-sec-mnist}

\textbf{Hyper-parameters of BEBU}. BEBU is the basic algorithm of BEBU-UCB and BEBU-IDS. BEBU uses the same network-architecture as Bootstrapped DQN~\citep{bootstrap-2016}. The diffusion factor and other training parameters are set by following EBU paper~\citep{ebu-2019}. Details are summarized in Table~\ref{tab:hyper-maze}.

\begin{table}[h!]
\small
  \caption{Hyper-parameters of BEBU for MNIST-Maze}
  \label{tab:hyper-maze}
  \centering
  \begin{tabular}{p{0.18\columnwidth}p{0.18\columnwidth}p{0.54\columnwidth}}
    \toprule
    Hyperparameters & Value     & Description \\
    \midrule
    state space  & $28\times28\times2$ & Stacking two images sampled from MNIST dataset with labels according to the agent's current location. \\
    action space & 4 & Including left, right, up, and down.\\
    $K$          & 10 & Number of bootstrapped heads. \\
    network-architecture & conv(64,4,4) conv(64,3,1) dense$\{512,4\}_{k=1}^{K}$ & Using convolution (channels, kernel size, stride) layers first, then fully connected into $K$ bootstrapped heads. Each head has $512$ ReLUs and $4$ linear units.\\
    gradient norm & 10 & The gradient is clipped by 10. The gradient of each head is normalize by $1/K$ according to bootstrapped DQN.\\
    learning starts & 10000 & The agent takes actions according to the initial policy before learning starts. \\
    replay buffer size & 170 & A simple replay buffer is used to store episodic experience. \\
    training frequency & 50 & Number of action-selection step between successive gradient descent steps. \\
    $H$ & 200,000 & Total timesteps to train a single maze. \\
    target network update frequency & 2000 & The target-network is updated every 2000 steps.\\
    optimizer & Adam & Adam optimizer is used for training. Detailed parameters: $\beta_1=0.9$, $\beta_2=0.999$, $\epsilon_{\rm ADAM}=10^{-7}$.\\
    learning rate & 0.001 & Learning rate for Adam optimizer. \\
	$\epsilon$ & $\frac{(h-H)^2}{H^2}$ & Exploration factor. $H$ is the total timesteps for training, and $h$ is the current timestep. $\epsilon$ starts from 1 and is annealed to 0 in a quadratic manner. \\
    $\gamma$ & 0.9 & Discount factor.\\
    $\beta$  & 1.0 & Diffusion factor of backward update. \\
    wall density & 30\%, 40\%, and 50\% & Proportion of walls in all locations of the maze. \\
    reward & -1 or 1000 & Reward is -1 when bumping into a wall, and 1000 when reaching the goal. \\
    stochasticity & 10\% & Has a probability of 10\% to arrive in the adjacent locations when taking an action. \\
    evaluation metric & $l_{\rm rel}=l_{\rm agent}/l_{\rm best}$ & Ratio between length of the agent's path and the best length. \\
    \bottomrule
  \end{tabular}
\end{table}

\textbf{Hyper-parameters of BEBU-UCB}. BEBU-UCB uses the upper-bound of $Q$-values to select actions. In particular, $a=\arg\max_{a\in\mathcal{A}}{[\mu(s,a)+\lambda_{\rm ucb}\sigma(s,a)]}$, where $\mu(s,a)$ and $\sigma(s,a)$ are the mean and standard deviation of bootstrapped $Q$-values $\{Q^k(s,a)\}_{k=1}^{K}$. We use $\lambda_{\rm ucb}=0.1$ in our experiment. 

\textbf{Hyper-parameters of BEBU-IDS}. The action-selection in IDS~\citep{info-2019} follows the regret-information ratio as $a_t=\argmin_{a\in\mathcal{A}}\frac{\hat{\Delta}_t(s,a)^2}{I_t(s,a)}$, which balances the regret and exploration. $\hat{\Delta}_t(s,a)$ is the expected regret that indicates the loss of reward when choosing a suboptimal action $a$. IDS uses a conservative estimate of regret, namely, $\hat{\Delta}_t(s,a)=\max_{a'\in\mathcal{A}}u_t(s,a')-l_t(s,a)$, where $[l_t(s,a),u_t(s,a)]$ is the confidence interval of action-value function. In particular, $u_t(s,a)=\mu(s,a)+\lambda_{\rm ids}\sigma(s,a)$ and $l_t(s,a)=\mu(s,a)-\lambda_{\rm ids}\sigma(s,a)$, where $\mu(s,a)$ and $\sigma(s,a)$ are the mean and standard deviation of bootstrapped $Q$-values $\{Q^k(s,a)\}_{k=1}^{K}$. The information gain $I_t (a)$ measures the uncertainty of action-values by $I(s,a)=\log (1+\frac{\sigma(s,a)^2}{\rho(s,a)^2})+\epsilon_{\rm ids}$, where $\rho(s,a)$ is the variance of the return distribution, which can be measured by C51~\citep{dis-2017} in distributional RL and is a constant in ordinary $Q$-learning. We set $\lambda_{\rm ids}=0.1$, $\rho(s,a)=1.0$, and $\epsilon_{\rm ids}=10^{-5}$ for our experiment.

\textbf{Hyper-parameters of OB2I}. We set $\alpha_1$ and $\alpha_2$ to be $0.01$ for our experiments. We find that adding a normalizer to UCB-bonus $\mathbf{\tilde{B}}$ of the next-$Q$ value enables more stable performance. A similar technique was used in \citet{largescale-2019}. Specifically, we divide $\mathbf{\tilde{B}}$ by a running estimate of its standard deviation. Since the UCB-bonuses for next-$Q$ are typically different among the $Q$-heads, such a normalization allows $Q$-networks to have a smooth and stable update.

\subsection{Atari games}\label{app-sec-atari}

\textbf{Hyper-parameters of BEBU}. We adopt the same basic setting of the Atari environment as \cite{DQN-2015} and \cite{ebu-2019}. We summarize the details to Table~\ref{tab:hyper-atari}.

\begin{table}[h!]
\small
  \caption{Hyper-parameters of BEBU for Atari games}
  \label{tab:hyper-atari}
  \centering
  \begin{tabular}{p{0.18\columnwidth}p{0.18\columnwidth}p{0.54\columnwidth}}
    \toprule
    Hyperparameters & Value     & Description \\
    \midrule
    state space  & $84\times84\times4$ & Stacking 4 recent frames as the input to network.\\
	action repeat & 4 & Repeating each action 4 times. \\
    $K$          & 10 & The number of bootstrapped heads. \\
    network-architecture & conv(32,8,4) conv(64,4,2) conv(64,3,1) dense$\{512,|\mathcal{A}|\}_{k=1}^{K}$ & Using convolution(channels, kernel size, stride) layers first, then fully connected into $K$ bootstrapped heads. Each head has $512$ ReLUs and $|\mathcal{A}|$ linear units.\\
    gradient norm & 10 & The gradient is clipped by 10, and also be normalize by $1/K$ for each head by following bootstrapped DQN.\\
    learning starts & 50000 & The agent takes random actions before learning starts. \\
    replay buffer size & 1M & The number of recent transitions stored in the replay buffer. \\
    training frequency & 4 & The number of action-selection step between successive gradient steps. \\
    $H$ & 20M & Total frames to train an environment. \\
    target network update frequency & 10000 & The target-network is updated every 10000 steps.\\
    optimizer & Adam & Detailed Adam parameters: $\beta_1=0.9$, $\beta_2=0.999$, $\epsilon_{\rm ADAM}=10^{-7}$.\\
    mini-batch size & 32 & The number of training cases for gradient decent each time. \\
    learning rate & 0.00025 & Learning rate for Adam optimizer. \\
	initial exploration & 1.0 & Initial value of $\epsilon$ in $\epsilon$-greedy exploration. \\
	final exploration & 0.1 & Final value of $\epsilon$ in $\epsilon$-greedy exploration. \\
	final exploration frames & 1M & The number of frames that the initial value of $\epsilon$ linearly annealed to the final value. \\
    $\gamma$ & 0.99 & Discount factor.\\
    $\beta$  & 0.5 & Diffusion factor of backward update. \\
    $\epsilon_{eval}$ & 0.05 & Exploration factor in $\epsilon$-greedy for evaluation.\\
    evaluation policy & ensemble vote & The same evaluation method as in Bootstrapped DQN \cite{bootstrap-2016}.\\
	evaluation length & 108000 & The policy is evaluated for 108000 steps. \\
	evaluation frequency & 100K & The policy is evaluated every 100K steps.\\
	max no-ops & 30 & Maximum number no-op actions before an episode starts. \\
    \bottomrule
  \end{tabular}
\end{table}

\textbf{Hyper-parameters of BEBU-UCB}. BEBU-UCB selects actions by $a=\arg\max_{a\in\mathcal{A}}{[\mu(s,a)+\lambda_{\rm ucb}\sigma(s,a)]}$. The detail is given in Appendix~\ref{app-sec-mnist}. We use $\lambda_{\rm ucb}=0.1$ in our experiment by searching coarsely. 

\textbf{Hyper-parameters of BEBU-IDS}. The action-selection follows the regret-information ratio as $a_t=\argmin_{a\in\mathcal{A}}\frac{\hat{\Delta}_t(s,a)^2}{I_t(s,a)}$. See detail in Appendix~\ref{app-sec-mnist}. We use $\lambda_{\rm ids}=0.1$, $\rho(s,a)=1.0$ and $\epsilon_{\rm ids}=10^{-5}$ in our experiment by searching coarsely. 

\textbf{Hyper-parameters of OB2I}. We set $\alpha_1$ and $\alpha_2$ to the same value of $0.5\times 10^{-4}$. The UCB-bonus $\mathbf{\tilde{B}}$ for the next-$Q$ value is normalized by dividing a running estimate of its standard deviation to have a stable performance.

\textbf{Implementation of Bayesian-DQN.} Since Bayesian-DQN is not evaluated in the whole Atari suite, we adopt the official release code in \url{https://github.com/kazizzad/BDQN-MxNet-Gluon} and make two modification for a fair comparison. (1) We add the 30 no-op evaluation mechanism, which we use to evaluate OB2I and other baselines in our work. (2) We set the frame-skip to 4 to be consistent with our baselines. We remark that inconsistency still exists since the original implementation of Bayesian-DQN is based on MX-Net Library, while OB2I and other baselines are implemented with Tensorflow. We release the modified code in \url{https://github.com/review-anon/Bayesian-DQN}.

\textbf{Results of DQN, UBE, BootDQN, Noisy-Net, and BootDQN-IDS.} These methods have been evaluated by the whole Atari suite. We directly adopt the scores reported in the corresponding articles \citep{DQN-2015,uncer-2018,bootstrap-2016,noise2-2018,info-2019}. However, we remark that inconsistency in the comparison exists since (1) UBE, BootDQN, and BootDQN-IDS use double Q-learning, and (2) Noisy-Net uses both the double Q-learning and dueling networks, in their original implementations. (3) In contrast, DQN, OB2I and BEBU-based baselines all use the standard Q-learning without advanced techniques.

\clearpage
\section{Visualizing OB2I}\label{app-vis-OB2I}

OB2I uses the UCB-bonus that indicates the disagreement of bootstrapped $Q$-estimates to measure the uncertainty of $Q$-functions. The state-action pairs with high UCB-bonuses signify the bottleneck positions or meaningful events. We provide visualization in several tasks to illustrate the effect of UCB-bonuses. Specifically, we choose \emph{Mnist-maze} and two Atari games \emph{RoadRunner} and \emph{Mspacman} to analyze. 

\subsection{MNIST-maze}\label{app-vis-OB2I-maze}

Figure~\ref{fig:maze-vis} illustrates the UCB-bonus in four randomly generated mazes. The mazes in Figure~\ref{fig:maze-vis1} and \ref{fig:maze-vis2} have a wall-density of $40\%$. The mazes in Figure~\ref{fig:maze-vis3} and \ref{fig:maze-vis4} have a wall-density of $50\%$. The left of each figure shows the map of maze, where the black blocks represent the walls. We omit the MNIST representation of states in the illustrations for simplification. A trained OB2I agent starts at the upper-left, then takes actions to achieve the goal at bottom-right. The UCB-bonuses of state-action pairs along the agent's visitation trajectory are computed and illustrated on the right of each figure. The value is normalized to $0\sim1$ for visualization. We show the maximal value if the agent appears several times in the same location.  

The positions with UCB-bonuses that higher than 0 draw the path of the agent. The path is usually winding and includes positions beyond the shortest path because the state transition has stochasticity. The state-action pairs with high UCB-bonuses are typically the bottleneck positions in the path. In maze~\ref{fig:maze-vis1}, the agent slips from the right path in position $(4, 7)$ to $(4, 9)$. The state-action in position $(4, 8)$ produces high bonus to guide the agent back to the right path. In maze~\ref{fig:maze-vis2}, the bottleneck state in $(3, 2)$ has high bonus to avoid the agent from entering into the wrong side of the fork. The other two mazes also have bottleneck positions, like $(3, 3)$ in maze~\ref{fig:maze-vis3} and $(7, 6)$ in maze~\ref{fig:maze-vis4}. Visiting these important locations is crucial to reaching the goal. We remark that the UCB-bonus of OB2I encourages the agent to walk through these bottleneck positions correctly.

\begin{figure}[!h]
\centering
\subfigure[]{\includegraphics[width=2.5in]{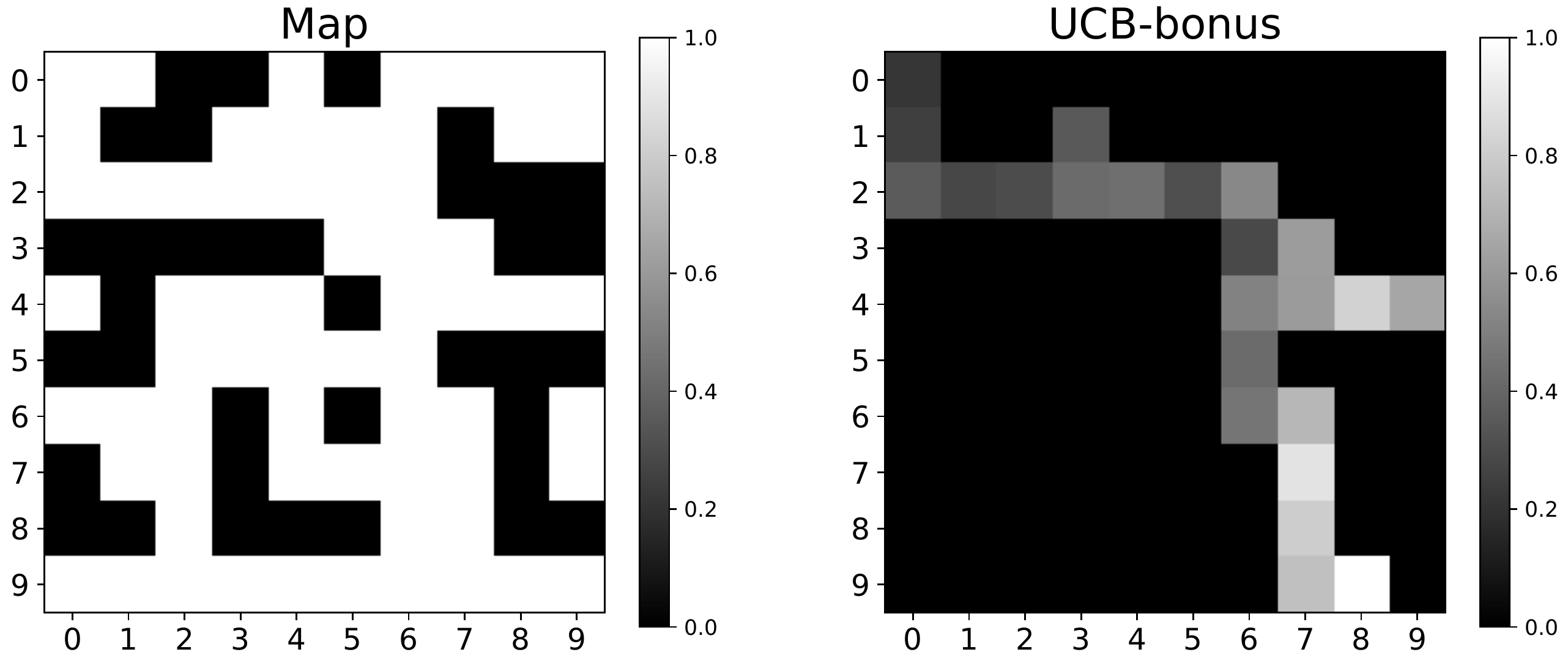}\label{fig:maze-vis1}}
\hspace{1.5em}
\subfigure[]{\includegraphics[width=2.5in]{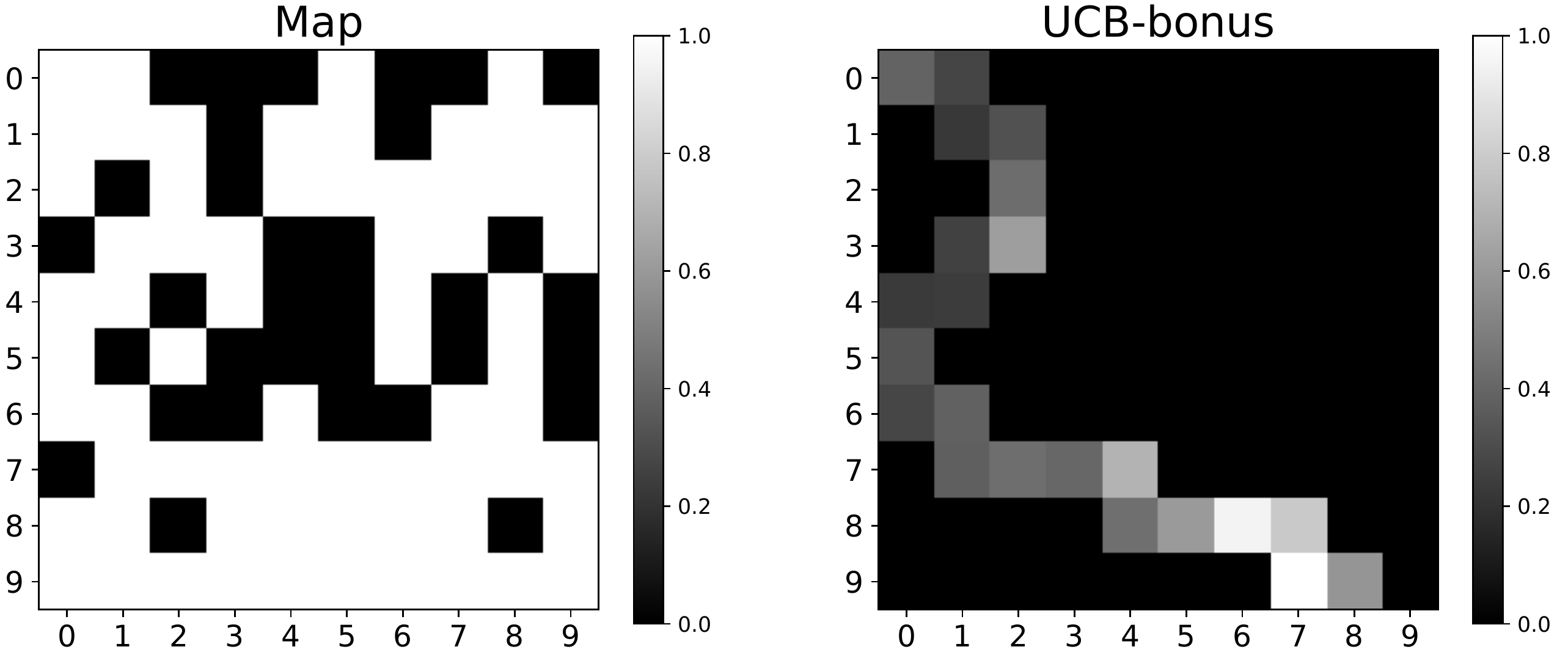}\label{fig:maze-vis2}}
\vspace{-1em}
\\
\subfigure[]{\includegraphics[width=2.5in]{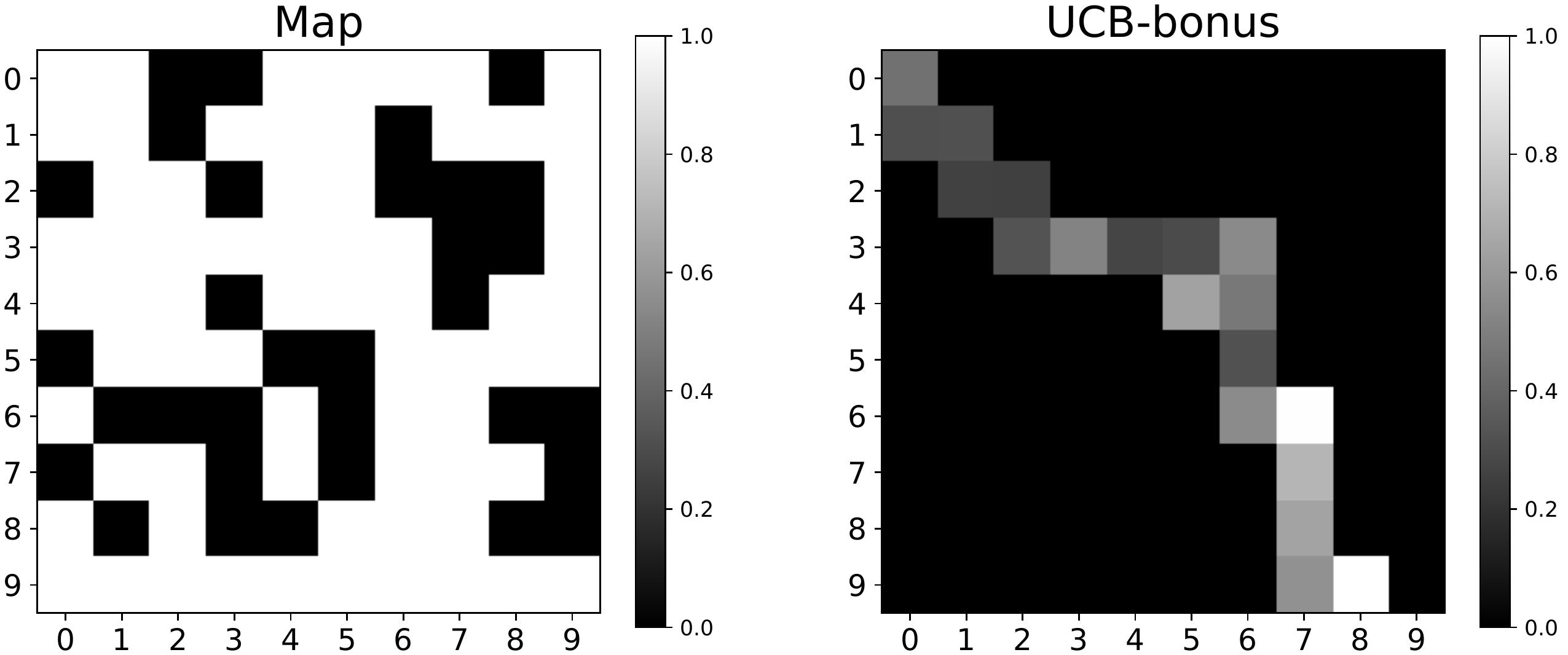}\label{fig:maze-vis3}}
\hspace{1.5em}
\subfigure[]{\includegraphics[width=2.5in]{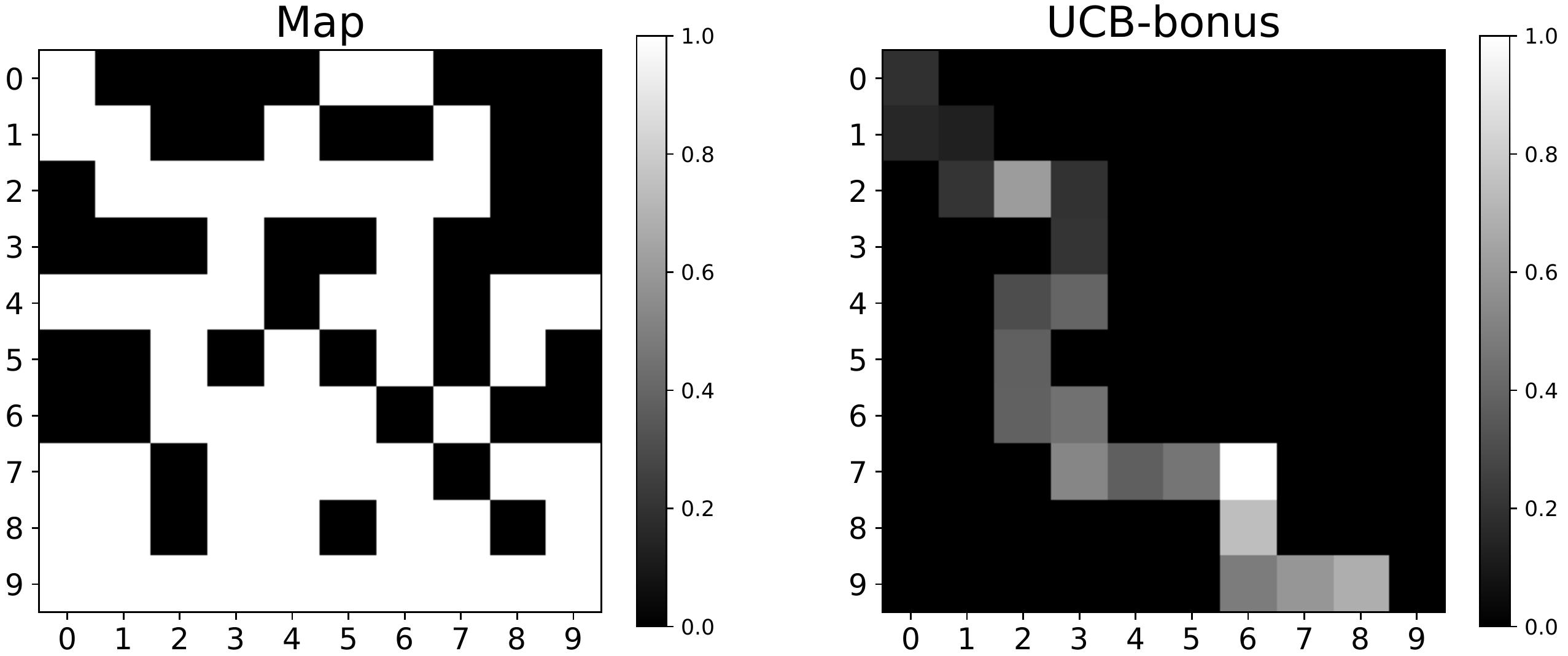}\label{fig:maze-vis4}}
\vspace{-1em}
\caption{Visualization of UCB-bonus in Mnist-maze}
\vspace{-1em}
\label{fig:maze-vis}
\end{figure}

\subsection{RoadRunner}\label{app-vis-OB2I-roadrunner}

In RoadRunner, the agent is chased by Wile E. Coyote and run endlessly to the left to escape. Picking the bird seeds on the street takes 100 points. Inducing Wile E. Coyote to be run over by a car takes 1000 points and also make the agent get rid of the danger of being chased up. In this task, the performance of OB2I is 90\% higher than that of BEBU. To illustrate how OB2I works, we use an OB2I agent to play this game for an episode and records the UCB bonus in all 1152 steps. Figure \ref{fig:visualize-roadrunner} shows the UCB bonus and the corresponding frames in 16 chosen spikes.

We find almost all spikes of UCB-bonus correspond to avoiding trucks and using trucks to get rid of Wile E. Coyote's chase (spike 2-14). The uncertainty is high with the emergence of truck because such a scenario rarely occurs. More importantly, utilizing the truck to get rid of Wile E. Coyote's chase has more uncertainty because the agent may get hit by the truck and lose its life. The UCB-bonus encourages the agent to learn skills that use the truck to gain advantages over the chaser and, hence, obtaining high scores. In addition, the agent eats bird seeds in spike 1. In spikes 15 and 16, the agent comes to a novel round. 

\begin{figure}[!h]
\center
\includegraphics[width=5.3in]{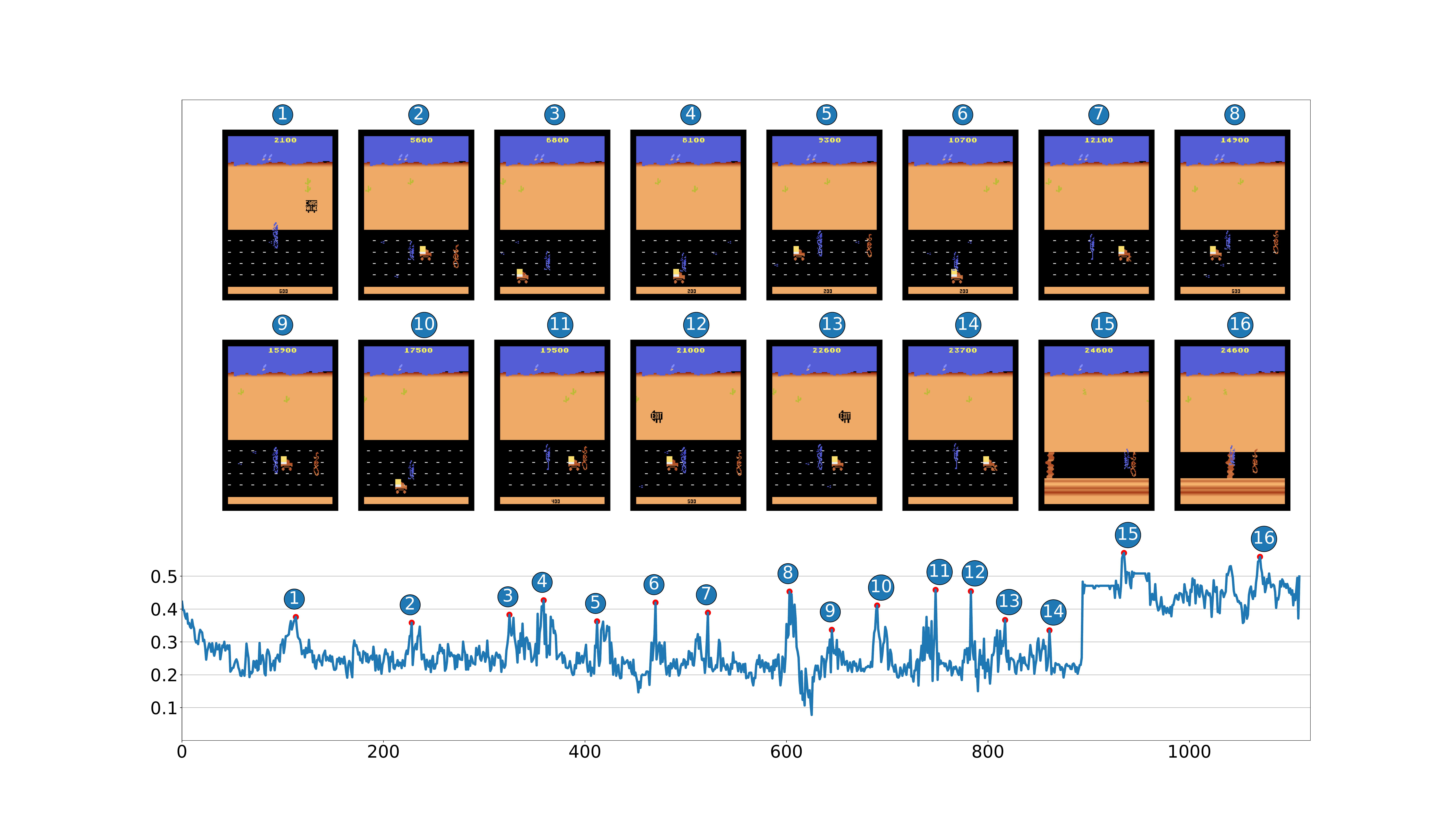}
\caption{Visualization of the UCB-bonus in RoadRunner. We further record frames after each spike, and the video is available at\\ \url{https://www.dropbox.com/sh/6ffgl9v53kkldau/AABzADhD9TW-9gjMYiJI-4jYa?dl=0}}
\label{fig:visualize-roadrunner}
\end{figure}




\subsection{MsPacman}\label{app-vis-OB2I-mspacman}

In MsPacman, the agent earns points by avoiding monsters and eating pellets. Eating an energizer causes the monsters to turn blue, allowing them to be eaten for extra points. We use a trained OB2I agent to interact with the environment for an episode. Figure~\ref{fig:visualize-mspacman} shows the UCB bonus in all 708 steps. We choose 16 spikes to visualize the frames. The spikes of exploration bonuses correspond to meaningful events for the agent to get rewards: starting a new scenario (1,2,9,10), changing direction (3,4,13,14,16), eating energizer (5,11), eating monsters (7,8,12), and entering the corner (6,15). These state-action pairs with high UCB-bonuses make the agent explore the environment efficiently. 

\begin{figure}[!h]
\center
\includegraphics[width=5.3in]{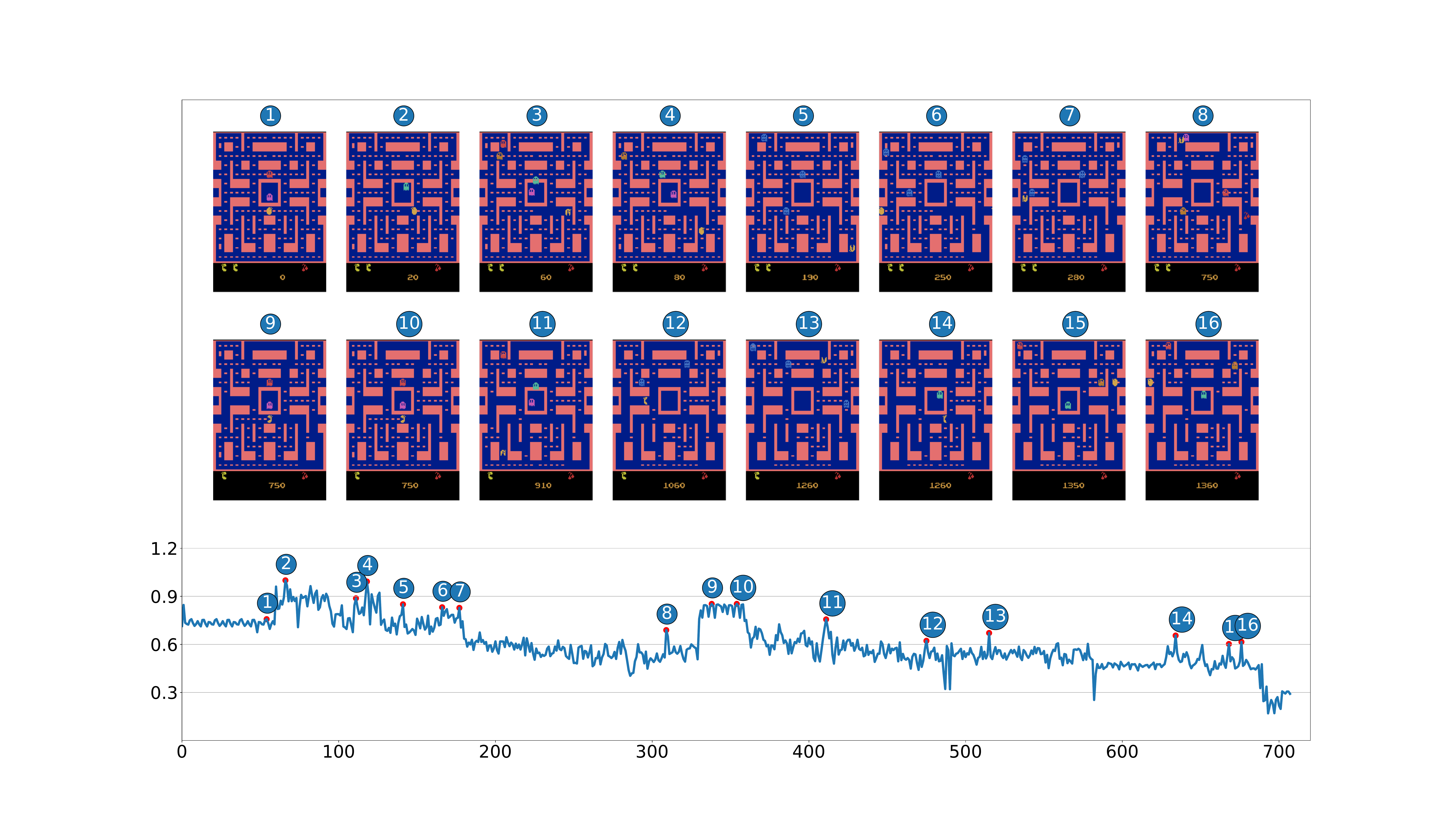}
\caption{Visualization of the UCB-bonus in MsPacman. We further record frames after each spike, and the video is available at\\ \url{https://www.dropbox.com/sh/6ffgl9v53kkldau/AABzADhD9TW-9gjMYiJI-4jYa?dl=0}}
\label{fig:visualize-mspacman}
\end{figure}

\clearpage

\section{Raw Scores of all 49 Atari Games}\label{app-raw-score}

\begin{table}[h!]
  \caption{Raw scores for Atari games. Bold scores signify the best score out of all methods.}
  \vspace{0.5em}
  \label{tab:scores-atari}
  \centering
  \begin{spacing}{0.98}
  \begin{tabular}{l|rr|rrrr}
     \toprule
     & Random & Human & BEBU & BEBU-UCB & BEBU-IDS & OB2I \\
     \hline
	 Alien & 227.8 & 6,875.0   & \textbf{1,118.0} & 811.1 & 857.9 & 916.9 \\
	 Amidar & 5.8 & 1676.0     & 81.7 & \textbf{166.4} & 148.1 & 94.0 \\
	 Assault & 222.4 & 1,496.0 & 1,377.0 & \textbf{3,574.5} & 2,441.8 & 2,996.2 \\
	 Asterix & 210.0 & 8,503.0 & 2,315.0 & 2,709.3 & 2,433.9 & \textbf{2,719.0} \\
	 Asteroids & 719.1 & 13,157.0   & 962.8 & \textbf{1,025.0} & 868.8 & 959.9 \\
	 Atlantis & 12,850.0 & 29,028.0 & 3,020,500.0 & 3,191,600.0 & 3,144,440.0 & \textbf{3,146,300.0} \\
	 Bank Heist & 14.2 & 734.4 & 331.8 & 277.0 & 361.6 & \textbf{378.6} \\
	 Battle Zone & 2,360.0 & 37,800.0 & 5,446.4 & \textbf{16,348.8} & 10,520.0 & 13,454.5 \\
	 BeamRider & 363.9 & 5,775.0 & 2,930.0 & 3,208.3 & 3,391.0  & \textbf{3,736.7} \\
	 Bowling & 23.1 & 154.8      & 29.9 & 30.7 & \textbf{40.2} & 30.0 \\
	 Boxing & 0.1 & 4.3          & 72.4 & 68.3 & 69.8 & \textbf{75.1} \\
	 Breakout & 1.7 & 31.8       & \textbf{473.2} & 382.3 & 412.7 & 423.1 \\
	 Centipede & 2,090.9 & 11,963.0      & 2,547.2 & 2,377.9 & \textbf{3,328.4} & 2,661.8 \\
	 Chopper Command & 811.0 & 9,882.0   & 930.6 & 1,013.4 & 1,100.0 & \textbf{1,100.3} \\
	 Crazy Climber & 10,780.5 & 35,411.0 & 49,735.7 & 39,187.5 & 42,242.9 & \textbf{53,346.7} \\
	 Demon Attack & 152.1 & 3,401.0 & 6,506.3 & 6,840.4 & \textbf{7,080.0} & 6,794.6 \\
	 Double Dunk & -18.6 & -15.5 & -18.9 & \textbf{-16.5} & -17.0 & -18.2 \\
	 Enduro & 0.0 & 309.6        & 504.1 & 697.8 & 513.6 & \textbf{719.0} \\
	 Fishing Derby & -91.7 & 5.5 & -56.7 & -83.8 & \textbf{-53.3} & -60.1 \\
	 Freeway & 0.0 & 29.6        & 21.5 & 21.6 & 21.3 & \textbf{32.1} \\
	 Frostbite & 65.2 & 4,335.0  & 393.4 & 470.4 & 466.2 & \textbf{1,277.3} \\
	 Gopher & 257.6 & 2,321.0    & 4,842.6 & \textbf{7,211.8} & 7,171.5 & 6,359.5 \\
	 Gravitar & 173.0 & 2,672.0  & 256.1 & 321.0 & 283.3 & \textbf{393.6} \\
	 H.E.R.O & 1,027.0 & 25,763.0 & 2,951.4 & 2,905.0 & 3,059.4 & \textbf{3,302.5} \\
	 Ice Hockey & -11.2 & 0.9     & -5.4 & -6.5 & -4.6 & \textbf{-4.2} \\
	 Jamesbond & 29.0 & 406.7     & \textbf{650.0} & 360.3 & 302.1 & 434.3 \\
	 Kangaroo & 52.0 & 3,035.0    & 3624.2 & 2,711.1 & \textbf{4,448.0} & 2,387.0 \\
	 Krull & 1,598.0 & 2,395.0    & 15,716.7 & 11,499.0 & 10,818.0 & \textbf{45,388.8} \\
	 Kung-Fu Master & 258.5 & 22,736.0   & 56.0 & 20,738.9 & \textbf{26,909.7} & 16,272.2 \\
	 Montezuma's Revenge & 0.0 & 4,376.0 & 0.0 & 0.0 & 0.0 & \textbf{0.0} \\
	 Ms. Pacman & 307.3 & 15,693.0       & 1,723.8 & 1,706.8 & 1,615.5 & \textbf{1,794.9} \\
	 Name This Game & 2,292.3 & 4,076.0  & 8,275.3 & 6,573.9 & \textbf{8,925.0} & 8,576.8 \\
	 Pong & -20.7 & 9.3            & 18.1 & 18.5 & 17.2 & \textbf{18.7} \\
	 Private Eye & 24.9 & 69,571.0 & 1,185.8 & 1,925.2 & 1,897.1 & 1,174.1 \\
	 Q*Bert & 163.9 & 13,455.0     & 3,588.4 & 3,783.2 & 3,696.0 & \textbf{4,275.0} \\
	 River Raid & 1,338.5 & 13,513.0 & 3,127.5 & \textbf{3,617.7} & 3,169.1 & 2,926.5 \\
	 Road Runner & 11.5 & 7,845.0 & 11,483.0 & 20,990.7 & 17,281.4 & \textbf{21,831.4} \\
	 Robotank & 2.2 & 11.9        & 10.3 & 13.3 & 10.7 & \textbf{13.5} \\
	 Seaquest & 68.4 & 20,182.0   & 447.0 & \textbf{492.3} & 332.4 & 332.1 \\
	 Space Invaders & 148.0 & 1,652.0 & 814.4 & 782.2 & 794.7 & \textbf{904.9} \\
	 Star Gunner & 664.0 & 10,250.0   & 1,467.2 & 1,201.5 & 1,158.9 & \textbf{1,290.2} \\
	 Tennis & -23.8 & -8.9          & -1.0 & -2.0 & -1.0 & \textbf{-1.0} \\
	 Time Pilot & 3,568.0 & 5,925.0 & 2,622.1 & 3,321.2 & 1,950.6 & \textbf{3,404.5} \\
	 Tutankham & 11.4 & 167.6      & 167.0 & 151.0 & 80.5 & \textbf{297.0} \\
	 Up and Down & 533.4 & 9,082.0 & \textbf{5,954.8} & 4,530.2 & 4,619.7 & 5,100.8 \\
	 Venture & 0.0 & 1,188.0       & 42.9 & 3.4 & \textbf{150.0} & 16.1 \\
	 Video Pinball & 16,256.9 & 17,298.0 & 26,829.6 & 48,959.1 & 58,398.3 & \textbf{80,607.0} \\
	 Wizard of Wor & 563.5 & 4,757.0 & 810.8 & \textbf{1,316.7} & 578.2 & 480.7 \\
	 Zaxxon & 32.5 & 9,173.0 & 1,587.5 & 2,104.8 & 1,594.2 & \textbf{2,842.0} \\
     \bottomrule
  \end{tabular}
  \end{spacing}
\end{table}

\clearpage
\section{Performance Comparison}\label{app-raw-score-comp}

We use the relative scores as 
\[\frac{\rm Score_{Agent}-Score_{Baseline}}{\rm \max\{Score_{human},Score_{baseline}\}-Score_{random}}\] 
to compare OB2I with baselines. The results of OB2I comparing with BEBU, BEBU-UCB, and BEBU-IDS is shown in Figure~\ref{fig:compare-1}, Figure~\ref{fig:compare-2}, and Figure~\ref{fig:compare-3}, respectively.

\begin{figure}[!h]
\vspace{-1em}
\center
\includegraphics[width=5.5in]{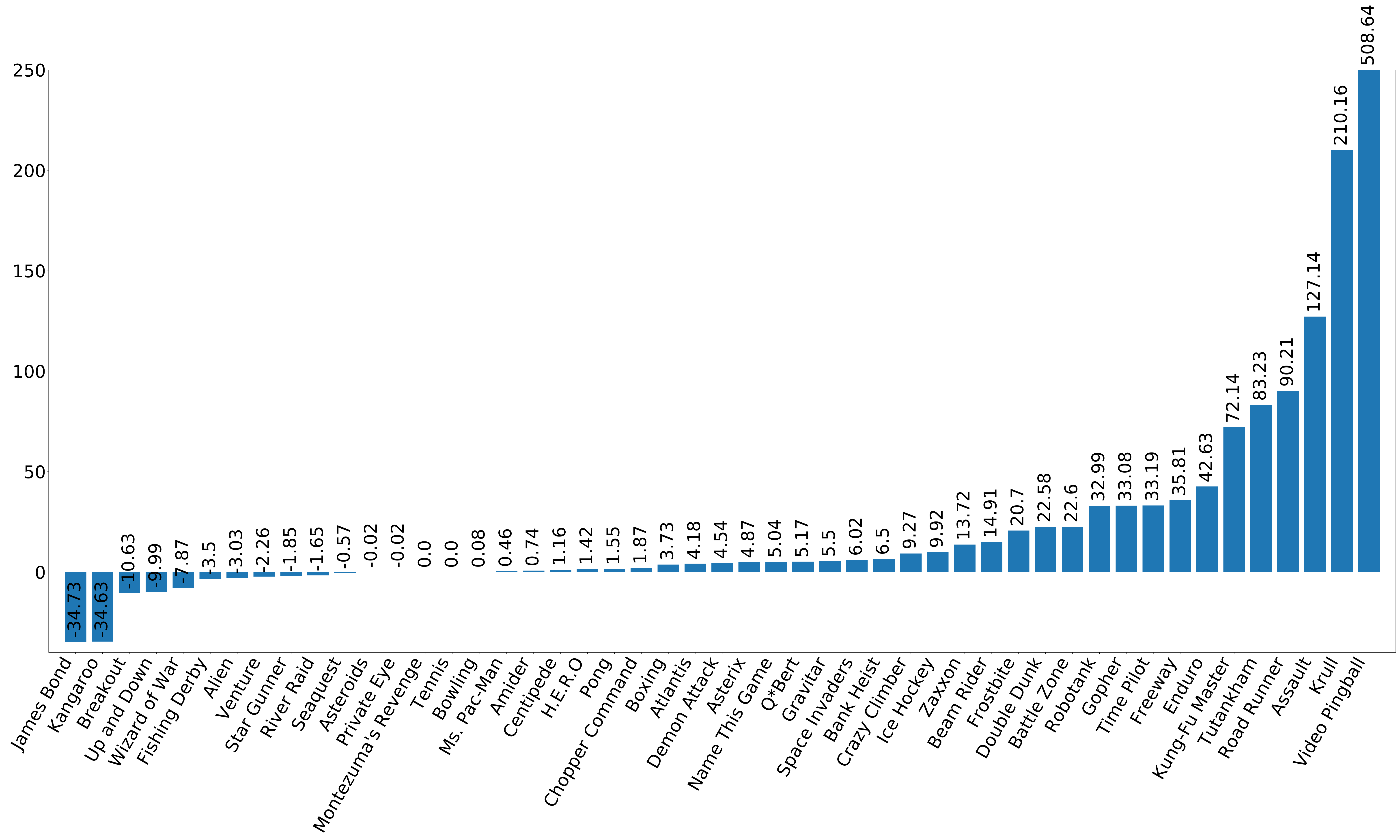}
\caption{Relative score of OB2I compared to BEBU in percents (\%).}
\vspace{-1em}
\label{fig:compare-1}
\end{figure}

\begin{figure}[!h]
\center
\includegraphics[width=5.5in]{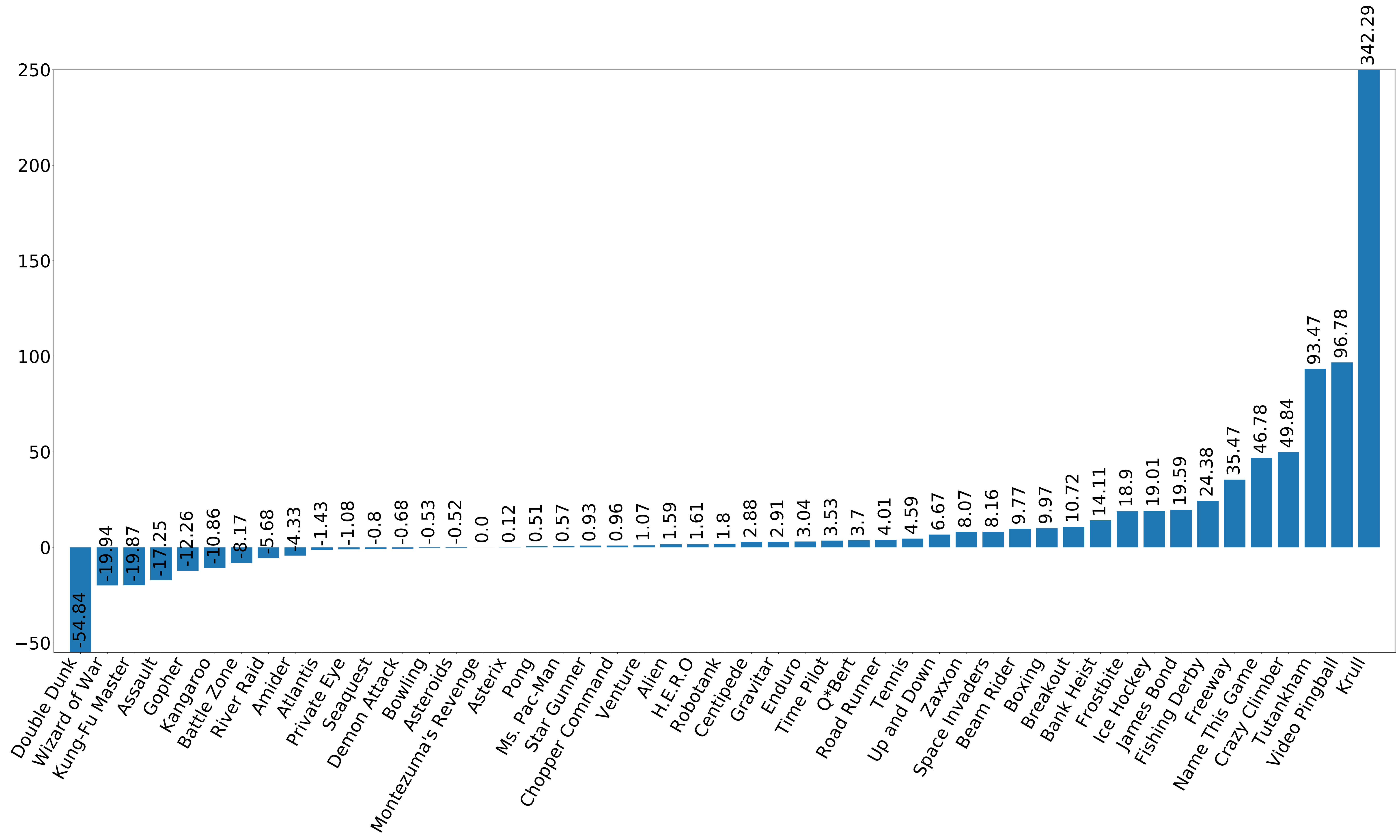}
\caption{Relative score of OB2I compared to BEBU-UCB in percents (\%).}
\vspace{-1em}
\label{fig:compare-2}
\end{figure}

\begin{figure}[!h]
\center
\includegraphics[width=5.5in]{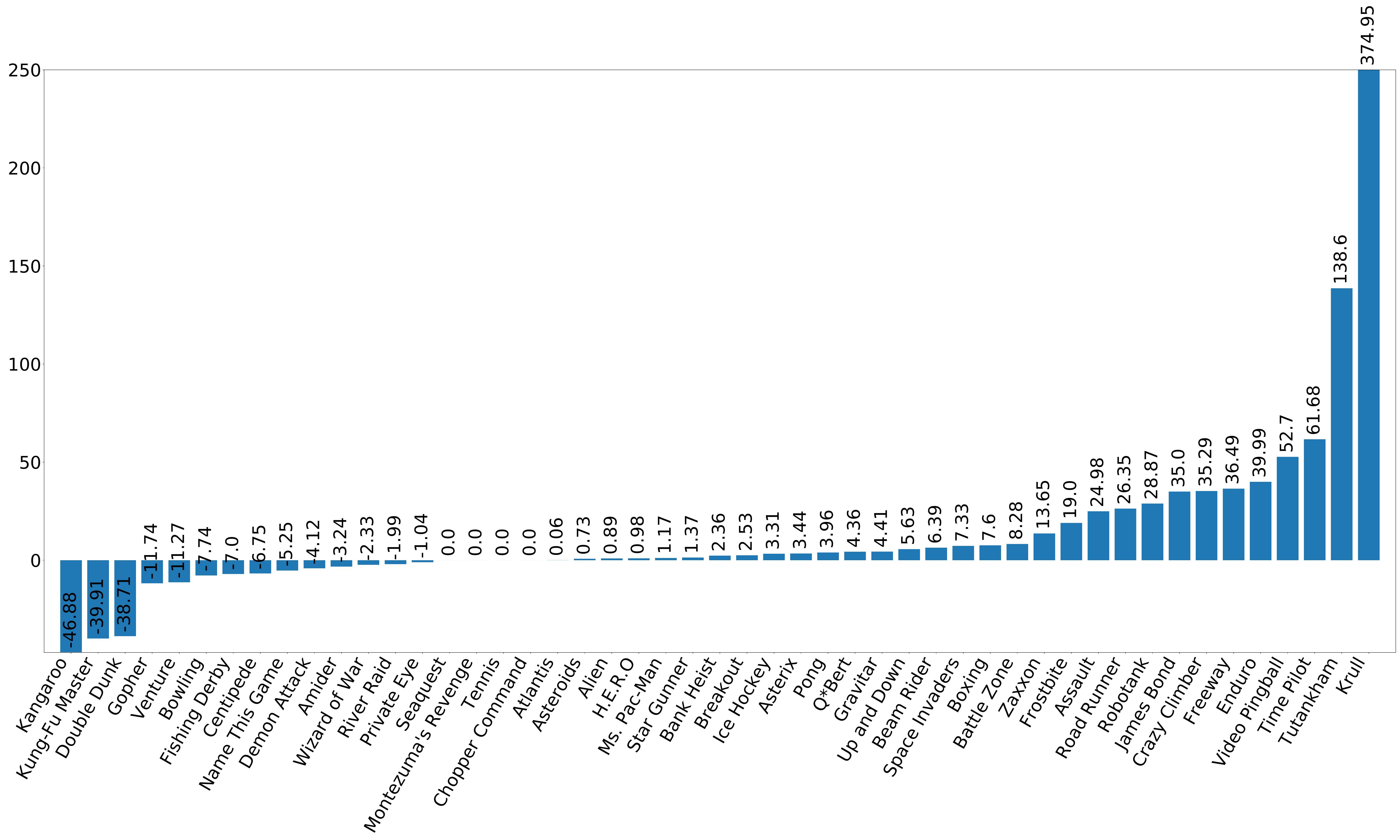}
\caption{Relative score of OB2I compared to BEBU-IDS in percents (\%).}
\label{fig:compare-3}
\end{figure}

\section{Failure Analysis}

\begin{figure}[b]
\center
\includegraphics[width=0.3\textwidth]{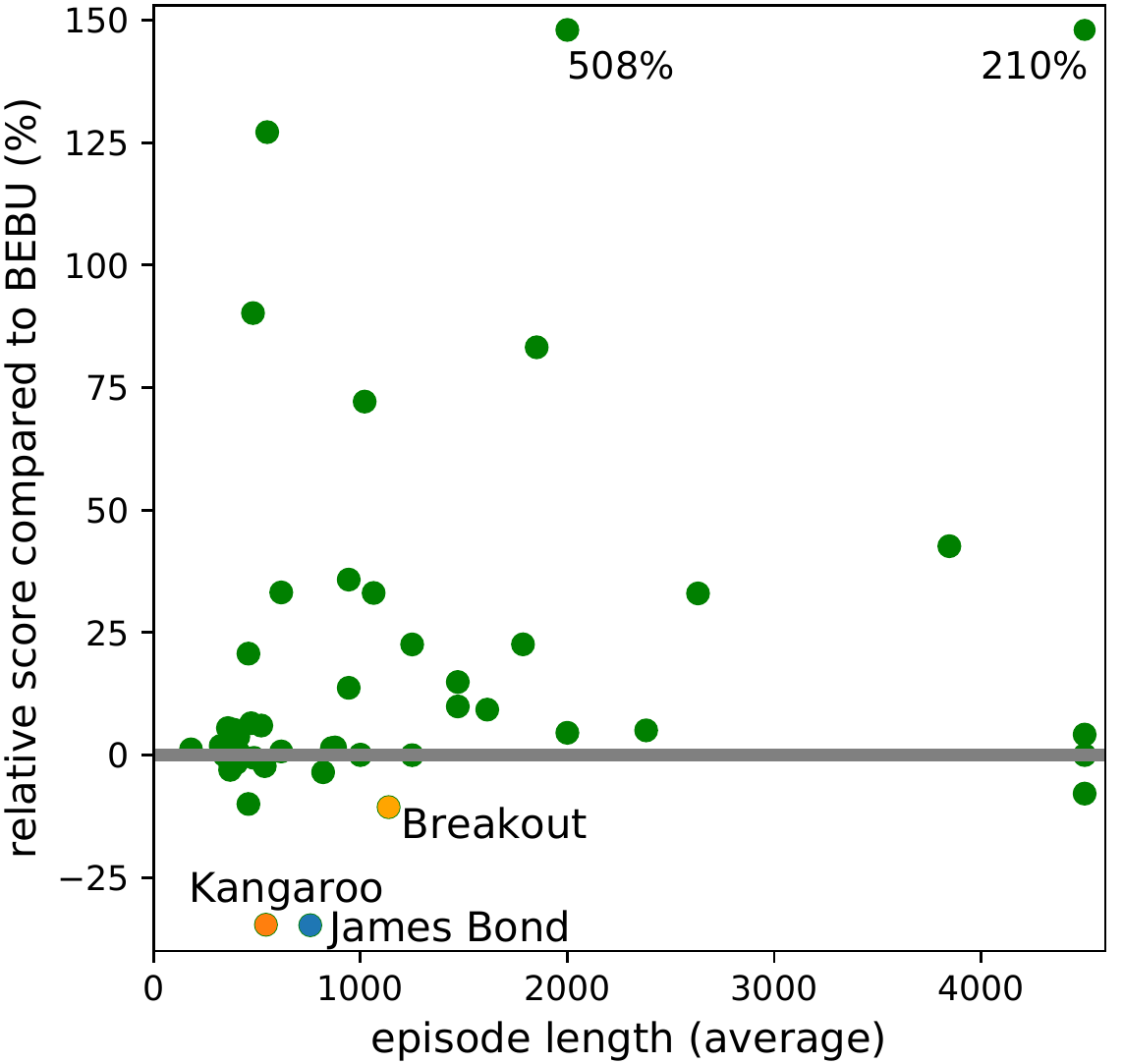}
\caption{ The relationship between Episode Length and Relative Scores}
\label{fig:failure}
\end{figure}

Our method does not have a good performance on Montezuma's Revenge (see Table \ref{tab:scores-montezuma}) because the epistemic uncertainty-based methods are not particularly tweaked for this domain. Meanwhile, IDS, NoisyNet and BEBU-based methods also fail on Montezuma's Revenge and score zero. Bootstrapped DQN achieves 100 points, which is also low and does not indicate successful learning in Montezuma's revenge. In contrast, the bonus-based methods achieve significantly higher scores on Montezuma's Revenge (e.g., RND achieves 8152 points). Nevertheless, according to \citet{bonux-2020} and Table~\ref{tab:scores-atari-sum}, NoisyNet and IDS significantly outperform several strong bonus-based methods evaluated by the mean and median scores of 49 Atari games. 

\begin{table}[!h]
\small
\caption{Comparison of scores in \emph{Montezuma's Revenge.}}
\label{tab:scores-montezuma}
\centering
\setlength{\tabcolsep}{1.3mm}{
\hspace{0.5em}
\begin{tabular}{c|cccc|cccc}
    \hline
    Frames & \multicolumn{4}{|c|}{200M} & \multicolumn{4}{|c}{20M} \cr
	\hline
    {~} & DQN & BootDQN & NoisyNet & BootDQN-IDS & BEBU & BEBU-UCB & BEBU-IDS & OB2I\cr
    \hline
    Scores & 0 & 100 & 3 & 0 & 0 & 0 & 0 & 0\cr
\hline
\end{tabular}}
\end{table}

Moreover, we find that the length of episode (or horizon) matters since OB2I propagates uncertainty within an episode. We visualize the connection between horizon and performance in Fig.~\ref{fig:failure}, where each point represents a game. We find that the games where OB2I is suboptimal typically have short horizons. In such games, propagating uncertainty does not bring much advantage, since it may be unnecessary.

Theoretically, BEBU (or BootDQN) instantiates Thompson sampling (with uninformative prior). As long as the prior is correctly specified, Thompson sampling attains the optimal Bayesian (average-case) regret. In contrast, OB2I instantiates optimism in the face of uncertainty (via UCB), which attains the optimal frequentist (worst-case) regret. In a few cases, OB2I may be overly conservative (with overly large UCB), since it aims to minimize the worst-case regret.

\section{Algorithmic Comparison}

\newcommand{\tabincell}[2]{\begin{tabular}{@{}#1@{}}#2\end{tabular}} 
\begin{table}[h]
\caption{Algorithmic comparison of the closely related works}
\hspace{0.1em}
\label{tab:alg-compare}
\centering
\begin{tabular}{l|c|c|c} 
\hline
& \tabincell{c}{Bonus or \\Posterior Variance} & Update Method & Uncertainty Characterization\\
\hline
EBU \citep{ebu-2019}  & - & backward update& -\\
Bootstrapped DQN \citep{bootstrap-2016} & bootstrapped & on-trajectory update & bootstrapped distribution\\
UBE \citep{uncer-2018} & closed form & on-trajectory update & posterior sampling\\
Bayesian-DQN \citep{bayesian-dqn} & closed form & on-trajectory update & posterior sampling\\
LSVI-UCB \citep{jin-2019} & closed form & backward update & optimism\\
BEBU (base of our work)  & bootstrapped & backward update & bootstrapped distribution\\
OB2I (ours) & bootstrapped & backward update & optimism\\ 
\hline
\end{tabular}
\end{table}

\end{document}